\documentclass{article}

\usepackage{microtype}
\usepackage{graphicx}
\usepackage{subfigure}
\usepackage{booktabs} 

\usepackage{hyperref}

\usepackage[utf8]{inputenc} 
\usepackage[T1]{fontenc}    
\usepackage{hyperref}       
\usepackage{url}            
\usepackage{amsfonts}       
\usepackage{nicefrac}       
\usepackage{todonotes}
\usepackage{makecell}
\usepackage{bigstrut}
\usepackage{pgfplots}
\usepackage{bm, color}
\usepackage{caption}
\usepackage{amsthm, amsmath}

\newtheorem{theorem}{Theorem}
\newtheorem{definition}{Definition}
\newtheorem{lemma}{Lemma}[section]

\usepackage{mathtools}

\usepackage[accepted]{whi2020}

\icmltitlerunning{Interpreting Spatially Infinite Generative Models}

\begin{document}

\twocolumn[
\icmltitle{Interpreting Spatially Infinite Generative Models}


\icmlsetsymbol{equal}{*}

\begin{icmlauthorlist}
\icmlauthor{Chaochao Lu}{cam,mpi}
\icmlauthor{Richard E.~Turner}{cam,msc}
\icmlauthor{Yingzhen Li}{msc}
\icmlauthor{Nate Kushman}{msc,dm}
\end{icmlauthorlist}

\icmlaffiliation{cam}{University of Cambridge, Cambridge, UK}
\icmlaffiliation{mpi}{Max Planck Institute for Intelligent Systems, T\"ubingen, Germany}
\icmlaffiliation{msc}{Microsoft Research, Cambridge, UK}
\icmlaffiliation{dm}{now at DeepMind}
\icmlcorrespondingauthor{Nate Kushman}{nate@kushman.org}

\icmlkeywords{interpretability, transparency}

\vskip 0.3in
]


\printAffiliationsAndNotice{}

\begin{abstract}
Traditional deep generative models of images and other spatial modalities can only generate fixed sized outputs.  The generated images have exactly the same resolution as the training images, which is dictated by the number of layers in the underlying neural network.  Recent work has shown, however, that feeding spatial noise vectors into a fully convolutional neural network enables both generation of arbitrary resolution output images as well as training on arbitrary resolution training images.  While this work has provided impressive empirical results, little theoretical interpretation was provided to explain the underlying generative process.  In this paper we provide a firm theoretical interpretation for infinite spatial generation, by drawing connections to spatial stochastic processes.  We use the resulting intuition to improve upon existing spatially infinite generative models to enable more efficient training through a model that we call an infinite generative adversarial network, or $\infty$-GAN.  Experiments on world map generation, panoramic images and texture synthesis verify the ability of $\infty$-GAN to efficiently generate images of arbitrary size.
\end{abstract}

\section{Introduction}

Generative modeling using neural networks has made significant progress over the last few years.  Especially dramatic has been the improvement of models for spatial domains such as images. Deep generative models, such as Generative-Adversarial Networks (GANs) \citep{goodfellow2014generative}, have been at the forefront of this effort due to their ability to generate sharp, photo-realistic images \citep{karras2018progressive,karras2018style,brock2018large}.
Despite this success, traditional generative models can only generate outputs of equal or smaller sizes than the original data on which they were trained. Furthermore, training on high resolution images is prohibitively expensive, as the size of the network typically must grow at least linearly with the resolution of the training samples.
In many domains, however, the resolution of the training samples is either extremely large, or effectively unbounded. Maps and earth imagery such as Landsat images are often very high resolution, and can cover arbitrarily large parts of the planet, resulting in images which are at least millions of pixels in each dimension.  Panoramic images can be very large in the horizontal dimension.  Texture images for graphics can be arbitrarily large depending on the object to which they are mapped, and medical images can often be $100,000\times100,000$ or larger~\citep{komura2018machine}.
Recent work \citep{jetchev2016texture,shaham2019singan} has shown that arbitrary sized images can be generated from a fixed sized neural network by feeding in a 2-d grid of noise vectors, i.e. a 3-d tensor of noise, rather than the single vector of noise used in traditional deep generative models.
While this work showed impressive empirical results, it did not provide a theoretical basis for these type of models.

In this paper we present a theoretical interpretation of infinite generative modelling based on the idea of {\em consistent} architectures, 
where the output generated by a given patch of noise will differ if that patch of noise is generated adjacent to another patch of noise, or on its own.  We show that past work has used {\em inconsistent architectures} which requires removing pixels from the boundary of image patches generated during the incremental generation process required to create large samples. This waste of computation can be significant, and we provide an in-depth analysis of the redundancy in terms of the fraction of discarded pixels.  Finally, we prove that by making small changes to the convolutional architectures we can ensure that the generator is a consistent transformation over spatial stochastic processes.  This ensures \emph{consistent} outputs, enabling the generated patches to be directly combined without removing pixels.

\begin{figure*}[t]
    \subfigure[A world map sample ($1395 \times 4595$, full-scale data image has size $1150 \times 2700$)]{
    \includegraphics[width = 0.735 \linewidth]{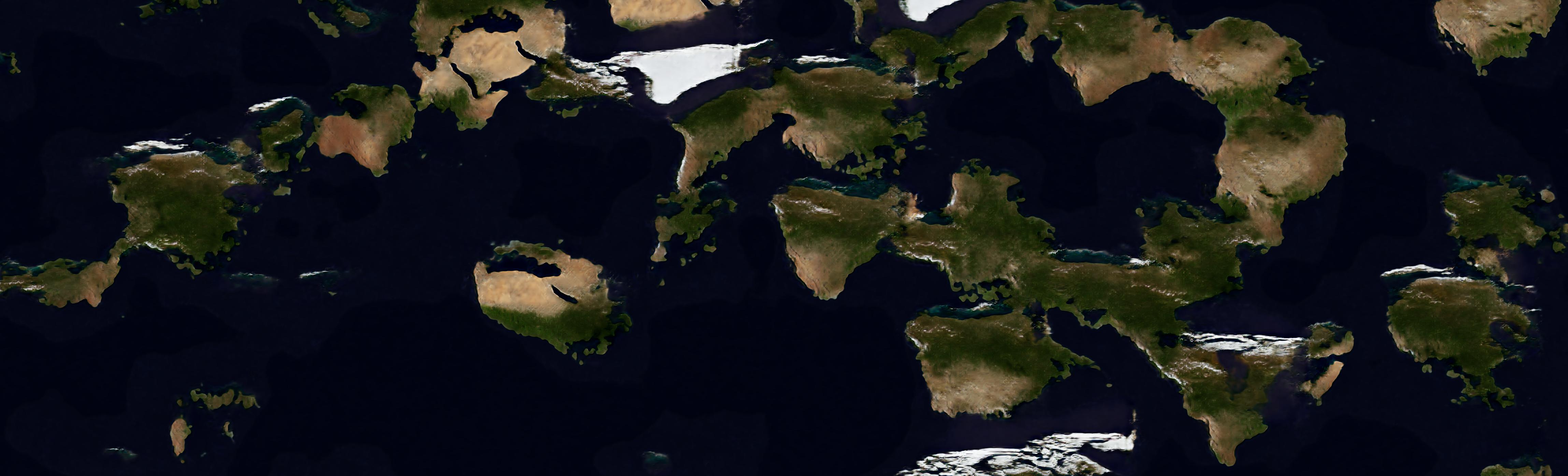}
    }
    \hfill
    \subfigure[Samples ($256\times256$)]{
    \includegraphics[width = 0.225 \linewidth]{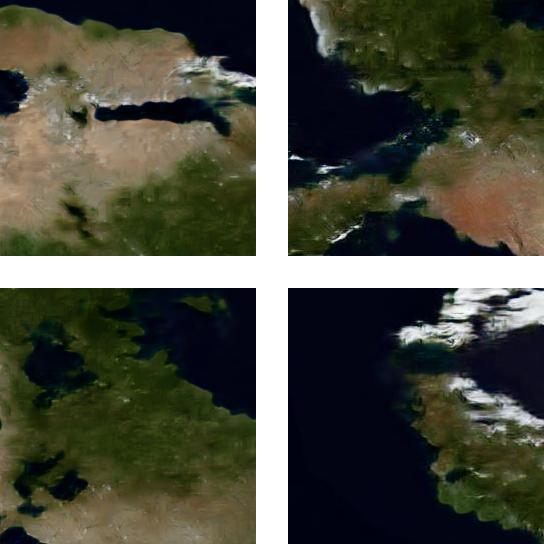}
    }
    \vspace{-5pt}
    \caption{
    The generated satellite world map from the $\infty$-GAN (trained on $64 \times 64$ patches). We see that the $\infty$-GAN can generate realistic looking continents with high-fidelity details.}
    \label{fig:generated_world_intro}
    \vspace{-12pt}
\end{figure*}
We use this intuition to design Infinite-GAN ($\infty$-GAN), which we test on texture images, panoramic views and a satellite map of the world. As shown in Figure \ref{fig:generated_world_intro}, our model can generate interesting new worlds without simply memorizing the Earth's world map, even when our model only observes $64 \times 64$ image patches during training.

\section{Theoretical Results}

\subsection{Redundancy in incremental generation with inconsistent network architectures}

To generate arbitrarily large images with convolutional neural networks (CNNs), the generator $G(\cdot)$ needs to be capable of handling noise inputs of arbitrary size. As convolution and up-samplinng operations are agnostic to the input sizes, a natural idea is to build a Spatial GAN-like generator $G(\cdot)$ with $K$ blocks of $\{\text{up-sample}, \text{conv } 3 \times 3 \text{ (stride 1, zero padding)} \}$ layers, in order to transform a latent tensor $\mathbf{z} \in \mathbb{R}^{h \times w \times C}$ into an image $\mathbf{x} \in \mathbb{R}^{2^K h \times 2^K w \times 3}$. 
Therefore $G(\cdot)$ is able to generate arbitrary large (but bounded) images \emph{in one shot} by increasing the size of $\mathbf{z}$ to any large (but also bounded) value.

However, \emph{incremental generation} is more suited to generating extremely large images by stitching small patches generated in a sequel. In such case the extension of $G(\cdot)$ to this task is less straight-forward.
More importantly, they produce \emph{inconsistent} image patches due to the use of padding. Given $\mathbf{z}_1 \subset \mathbf{z}_2$, the output image $G(\mathbf{z}_1)$ is not a sub-patch of $G(\mathbf{z}_2)$. Similarly when $\mathbf{z}_1 \cap \mathbf{z}_2 \neq \emptyset$, a native stitch of image patches $G(\mathbf{z}_1)$ and $G(\mathbf{z}_2)$ will result in tiling and/or boundary inconsistency artifacts.
An ad-hoc solution to this inconsistency issue is to stitch \emph{properly cropped} patches \citep{jetchev2016texture}. However, depending on the computational constraints on the maximum spatial size $(S, S)$ for one patch, it can result in discarding many pixels thus wasting a large amount of computation. This redundancy is characterized by the following result proved in Appendix \ref{sec:app_redunndancy}.
\begin{theorem}
\label{thm:redundancy}
Assume the stitched image has spatial size tending towards infinity. Then the fraction of discarded pixels is at least $4 ( \lfloor \frac{S}{2^K} \rfloor^{-1} - \lfloor \frac{S}{2^K} \rfloor^{-2} )$.
\vspace{-3pt}
\end{theorem}
As an example, consider $S = 4096$ and $K = 9$, then the redundancy figure in Theorem \ref{thm:redundancy} is $43.75\%$. This redundancy can be reduced by decreasing $K$, but then the network will fail to represent global features in a big image. On the other hand, although increasing $K$ makes the generator more flexible, the benefit is offset by inefficient generation process. 

Since one can keep generating patches in the incremental generation process, the output image is effectively unbounded thus \emph{infinite} dimensional. Mathematically this means incremental generation can only be achieved by considering \emph{consistent} transformations over \emph{spatial stochastic processes}; effectively the above incremental generation approach with inconsistent CNNs is one such transformation. Below we establish in theory a \emph{consistent} transform of spatial stochastic processes and discuss a more efficient CNN parameterization of such transformations.

\begin{figure*}[t]
\begin{minipage}{0.73\linewidth}
\captionof{table}{Summary of consistent and stationarity preserving transforms on (stationary) spatial processes. $^*$Proved for commonly used activation functions (sigmoid, ReLU, etc.). $^{**}$It transforms a strictly stationary spatial process to a cyclostationary spatial process.}
\centering
\label{tab:theory_summary}
\setlength{\tabcolsep}{2pt}
\scalebox{0.85}{
\begin{tabular}{ccc}
\toprule
 layer & consistent? & stationarity preserving? \\
\midrule
convolution (zero/constant padding) & No & No (not a valid transform) \\
convolution (no padding, stride 1) & Yes & Yes \\
pixel-wise non-linearity & Yes$^*$ & Yes \\
up-sampling (nearest, scale 2) & Yes & Yes$^{**}$ \\
up-sampling (bi-linear, scale 2, boundary crop) & Yes & Yes$^{**}$ \\
pixel normalization & Yes & Yes \\
\bottomrule 
\end{tabular}
}
\end{minipage}
\hfill
\begin{minipage}{0.25\linewidth}
\centering
\includegraphics[width=0.9\linewidth]{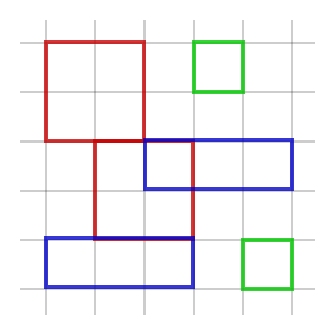}
\vspace{-0.15in}
\captionof{figure}{Visualizing cyclostationarity with period $(2, 1)$.}
\label{fig:patch_stationarity}
\end{minipage}
\vspace{-7pt}
\end{figure*}

\subsection{A convolutional neural network transform for spatial stochastic processes}


\paragraph{Consistent transforms for spatial stochastic processes}

Mathematically, using latent tensor of arbitrary sizes is equivalent to using a spatial stochastic process $\bm{z}(\cdot, \cdot) = \{ \bm{z}(i, j) \in \mathbb{R}^C \ | \ (i, j) \in \mathbb{Z}^2 \}$ as the latent input, e.g.~i.i.d.~Gaussian noise, i.e.~$\bm{z}(i, j) \overset{\text{i.i.d.}}\sim \mathcal{N}(0, \mathrm{I})$.
As computers cannot handle objects of infinite dimensions, in practice we first sample a ``noise patch'' $\bm{z}(J) = \{ \bm{z}(i, j) \ | \ (i, j) \in J \}$ with \emph{latent pixel index set} $J = \{1, ..., h\} \times \{1, ..., w\}$ of size $(h, w)$, then pass it through the generator and obtain an image $\bm{x}(I) = \{ \bm{x}(i, j) \in \mathbb{R}^3 \ | \ (i, j) \in I \} = G(\bm{z}(J))$ of size $(H, W)$ with \emph{image pixel index set} $I = \{1, ..., H\} \times \{1, ..., W \}$. For notation ease we write $\{a:b\} = \{a, ..., b \}$ for $a < b$, and when $a=b$, $\{a:b\}=\{a\}$.
Also we wish to have the following two desirable properties to allow \emph{incremental} generation of unbounded images from patches:
\begin{itemize}
\vspace{-10pt}
\setlength\itemsep{0em}
\setlength{\itemindent}{0em}
    \item Marginalization consistency: for any latent pixel index sets $J' \subset J$ of finite sizes (and the corresponding image pixel index sets $I' \subset I$), taking the patch with index set $I'$ from $\bm{x}(I) = G(\bm{z}(J))$ is equivalent to directly generating $\bm{x}(I') = G(\bm{z}(J'))$;
    \item Permutation invariance: generating two patches $\bm{x}(I), \bm{x}(I'')$ in the same image is invariant to the order of generation.
\vspace{-10pt}
\end{itemize}
If these two conditions are meet, then incremental image generation can be achieved by cutting the latent pixel index set $J$ into subsets $\{J ^n \ | \ \cup_n J^n = J, J^{n_1} \cap J^{n_2} \neq \emptyset \}$ with proper index subset overlap, sampling $\bm{z}(J)$. and stitching the individual patches $G(\bm{z}(J^n))$ together (without cropping) to obtain $\bm{x}(I) = \cup_n G(\bm{z}(J^n))$. Marginalization consistency ensures that $G(\bm{z}(J)) = \cup_n G(\bm{z}(J^n))$, and permutation invariance allows each patch $G(\bm{z}(J^n))$ to be generated in random order (provided an algorithm to retrieve the corresponding subset sample $\bm{z}(J^n)$ from $\bm{z}(J)$).

The two properties can be ensured by the \emph{consistency} requirement for $G(\cdot)$: denote the corresponding operator on stochastic processes as $\mathcal{O}_G$, then the output $\bm{x}(\cdot, \cdot) = \{ \bm{x}(i, j) \in \mathbb{R}^3 \ | \ (i, j) \in \mathbb{Z}^2 \} = \mathcal{O}_G(\bm{z}(\cdot, \cdot))$ also needs to be a spatial stochastic process. We prove in Appendix \ref{sec:consistency_proofs} the consistency results for popular CNN building blocks summarized in Table \ref{tab:theory_summary}. Specifically, padding is discarded to remove inconsistency at the boundary of each layer. Similarly the bi-linear up-sampling operation (scale 2) is adapted by cropping out the boundary pixels with edge width 1.
Our generative network (discussed below) only uses consistent transforms, and we have the following result.
\begin{theorem}
$\mathcal{O}_G$ is a consistent operator on spatial stochastic processes if $G(\cdot)$ is constructed using the consistent transforms presented in Table \ref{tab:theory_summary}.
\vspace{-4pt}
\end{theorem}

%
\paragraph{The stationarity pattern from CNN transforms}

A generative model can only see finite number of images of finite size in training, thus without further assumptions, the generator cannot generate realistic looking images of sizes that are larger than the maximum size of observed images. 

However, if the input $\bm{z}(\cdot, \cdot)$ in training time is a \emph{strictly stationary} spatial stochastic process, e.g.~i.i.d.~Gaussians, then for any index set $J \subset \mathbb{Z}^2$, the distribution of $\bm{z}(J)$ is shift-invariant for any shifting direction. Thus if the generator $G(\cdot)$ preserves stationarity on patch level, then during training it suffice to fix e.g.~$J = \{1: h \} \times \{1: w \}$, and match the distribution of $\bm{x}(I) = G(\bm{z}(J))$ to the distribution of training image patches of size $(H, W)$.

In math, this requires the generator to be able to transform a strictly stationary process into a cyclostationary process.
\begin{definition}
A spatial stochastic process $\bm{x}(\cdot, \cdot)$ is cyclostationary with period $(H , W)$ if for all $I = \{a : b\} \times \{c : d \} \subset \mathbb{Z}^2$ and $\tilde{I} = \{a + H : b + H\} \times \{c + W : d + W \}$, $\bm{x}(I) \overset{\text{d}}= \bm{x}(\tilde{I})$.
\vspace{-3pt}
\end{definition}
We visualize cyclostationarity with period $(2, 1)$ in Figure \ref{fig:patch_stationarity}; in this figure patches with bounding boxes in the same color have the same distribution.
Note that strictly stationary is cyclostationary with period $(1, 1)$; in general the pixels inside an $(H, W)$ patch do not necessarily have the same marginal distributions. 
In Appendix \ref{sec:stationarity_proofs} we investigate the \emph{stationarity preserving} properties, i.e.~given a cyclostationary input process, whether the output process is still cyclostationary, for CNN building blocks. Results are again summarized in Table \ref{tab:theory_summary}; using stationarity preserving transforms as CNN layers, we can prove the following result.
\begin{theorem}
$\mathcal{O}_G$ preserves cyclostationarity if $G(\cdot)$ is constructed using the stationarity preserving transforms presented in Table \ref{tab:theory_summary}.
\vspace{-3pt}
\end{theorem}
We note that using up-sampling of scale $2$ increases the stationarity period by $2$. For i.i.d.~Gaussian inputs, this means $\bm{x}(\cdot, \cdot)$ has stationarity period $(2^K, 2^K)$ if $G(\cdot)$ uses $K$ up-sampling layers.

\begin{table*}
 \caption{{\bf Satellite image FID scores} (mean $\pm$ std.) from a bootstrap estimate of FID (5 random crops). The $\infty$-GAN with the proposed multi-scale training significantly outperforms the baselines.}  \label{tab:world}
 \begin{center}
   \begin{tabular}{c|cccc}
 \toprule

     model / patch size &$32\times 32$&$64\times 64$&$128\times 128$&$256\times 256$\\
\midrule
Spatial GAN & $40.93 \pm 0.68$ & $43.75 \pm 0.33$ & $54.41 \pm 0.25$ & $71.69 \pm 0.25$ \\
\hline
PSGAN & $34.55 \pm 0.25$ & $48.77 \pm 0.20$ & $64.93 \pm 0.55 $  & $81.65 \pm 0.32$ \\
\hline
$\infty$-GAN - $\mathcal{D}_2$ Only & $39.87 \pm 0.35$ & $67.05 \pm 0.37$ & $95.12 \pm 0.91$ & $133.75 \pm 0.27$ \\
\hline
$\infty$-GAN - Full & $12.09 \pm 0.25$ & $24.18 \pm 0.09$ & $45.35 \pm 0.22$ & $56.60 \pm 0.31$ \\
\bottomrule
   \end{tabular}
 \end{center}
 \vspace{-15pt}
 \end{table*}

\begin{figure*}[t]
  \centering
  \subfigure[Original Data]{
    \includegraphics[width = 0.18 \linewidth]{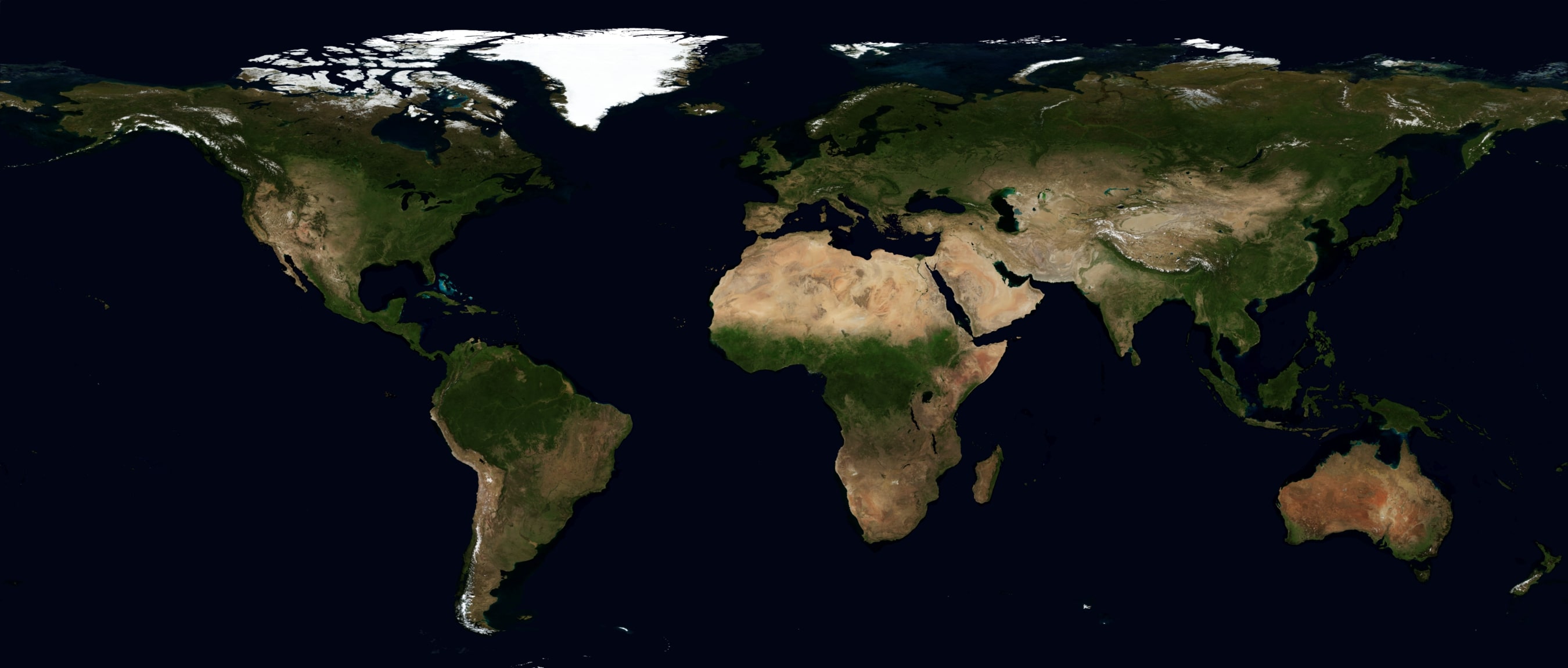}
    }
    \hfill
  \subfigure[Spatial GAN]{
    \includegraphics[width = 0.18 \linewidth]{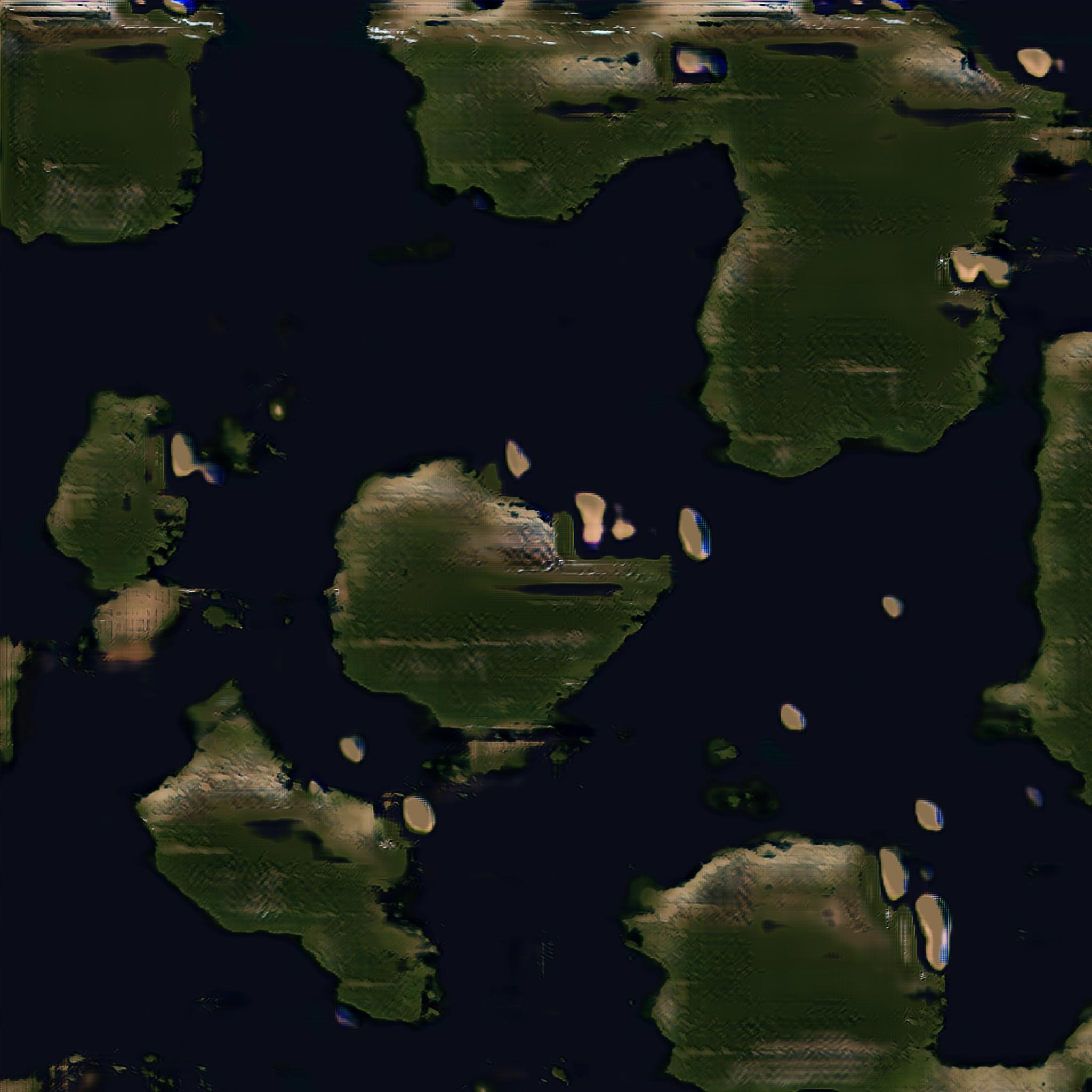}
    }
    \hfill
  \subfigure[PSGAN sample]{
    \includegraphics[width = 0.18 \linewidth]{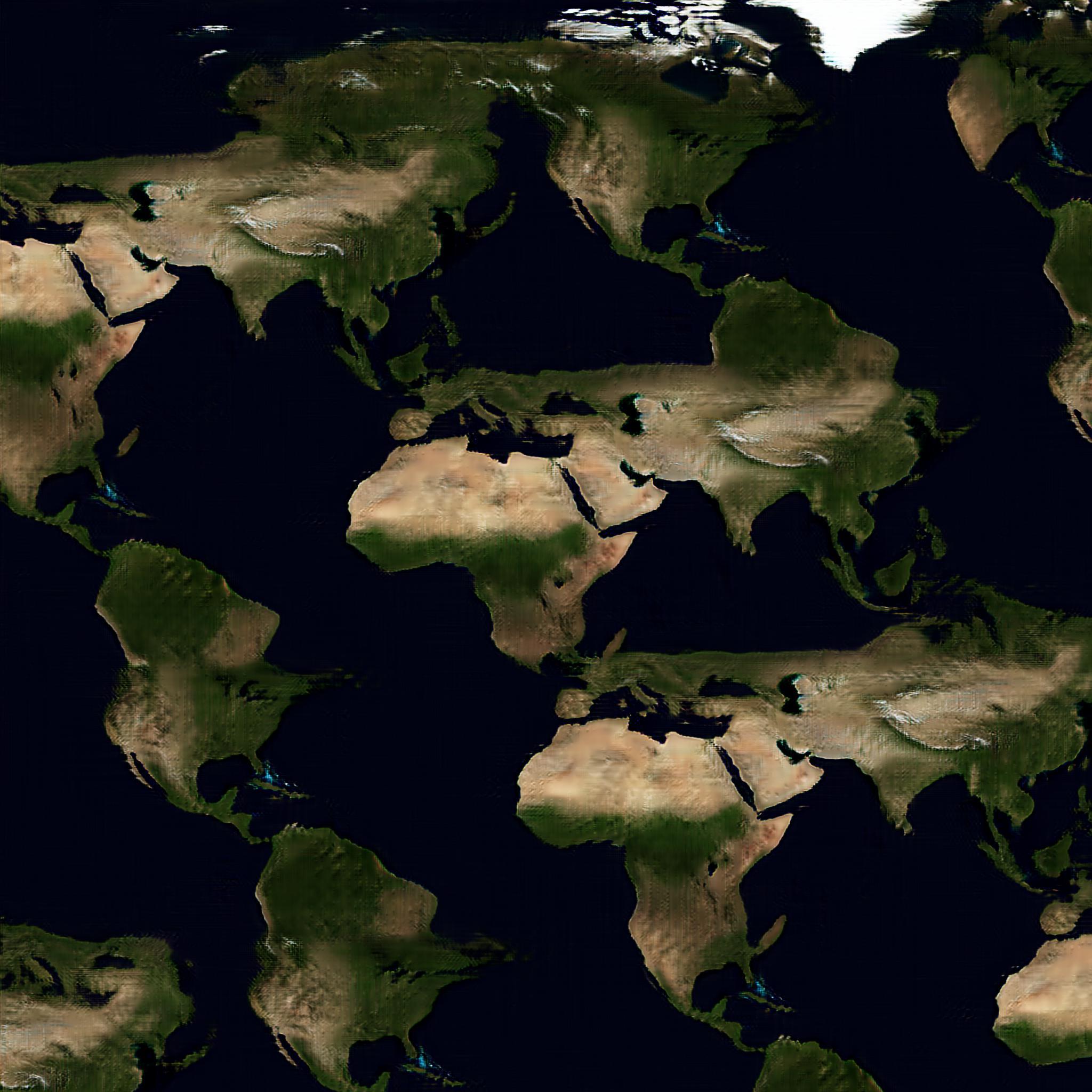}
    }
    \hfill
    \subfigure[$\mathcal{D}_2$ only model]{
    \includegraphics[width = 0.18 \linewidth]{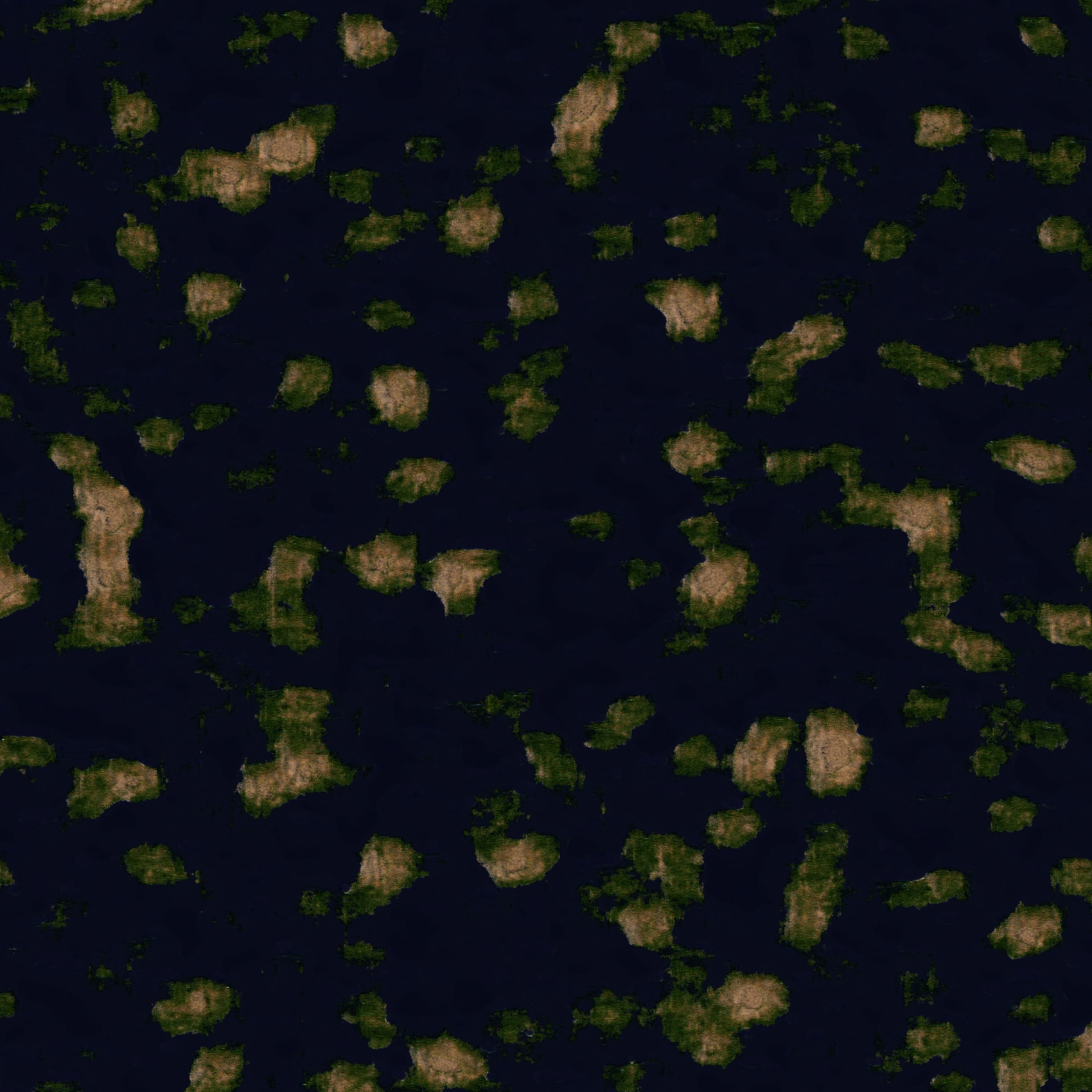}
    }
    \hfill
    \subfigure[$\infty$-GAN sample]{
    \includegraphics[width = 0.18 \linewidth]{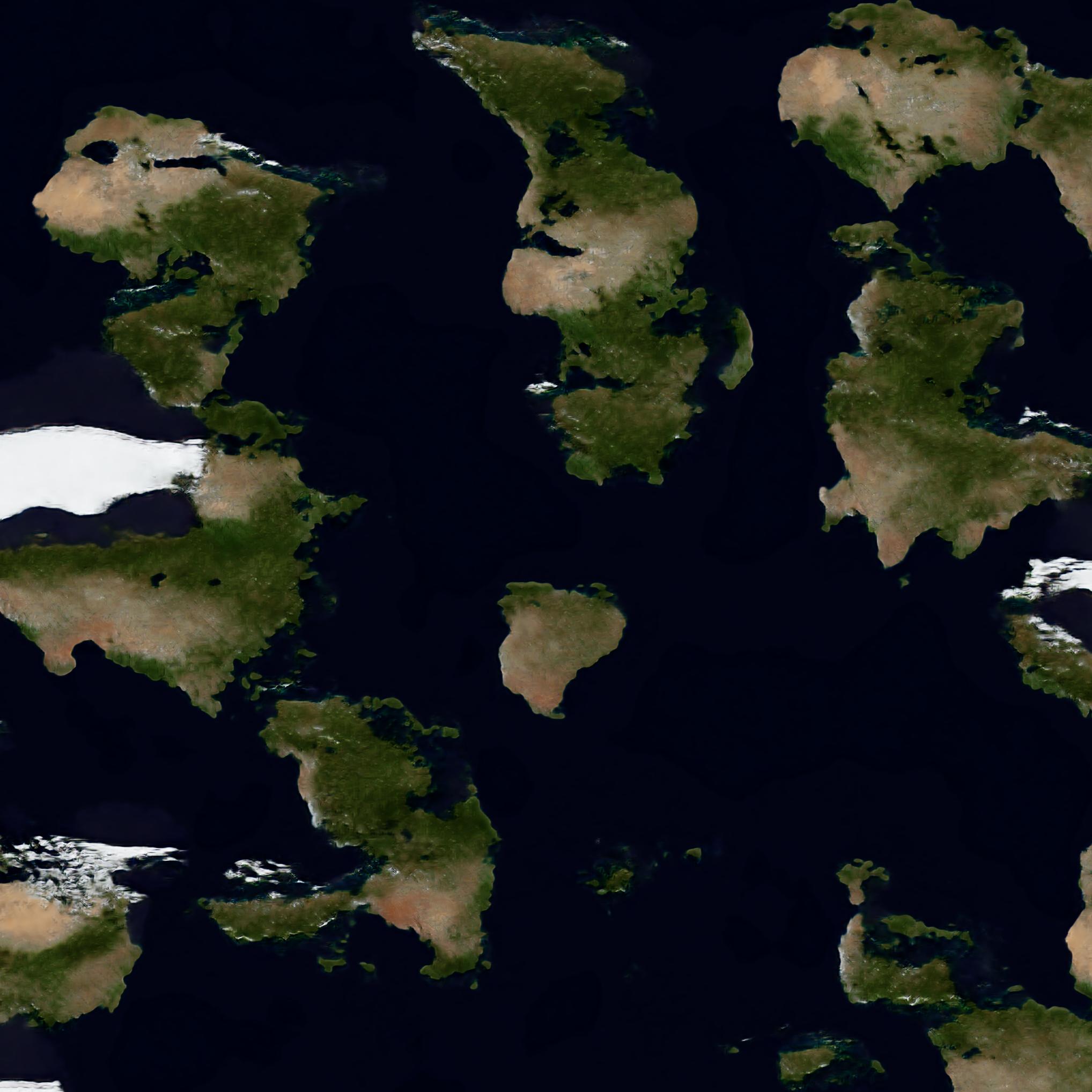}
    }\\
    \vspace{-5pt}
    \subfigure[Data patches]{
    \includegraphics[width = 0.18 \linewidth]{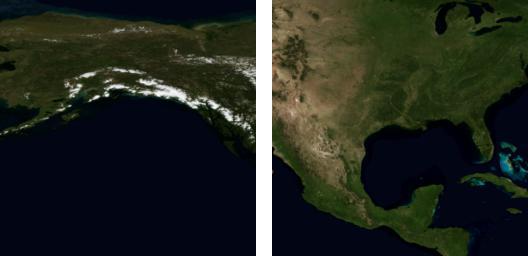}
    }
    \hfill
    \subfigure[Spatial GAN]{
    \includegraphics[width = 0.18 \linewidth]{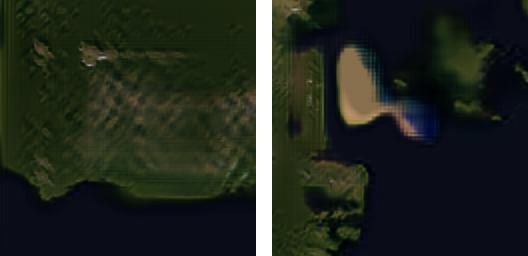}
    }
    \hfill
    \subfigure[PSGAN patches]{
    \includegraphics[width = 0.18 \linewidth]{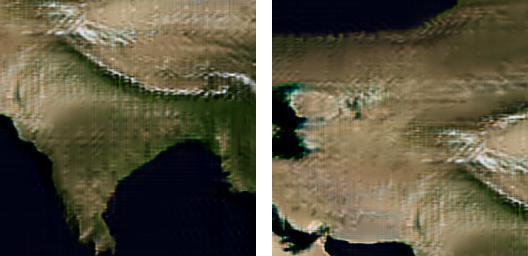}
    }
    \hfill
    \subfigure[$\mathcal{D}_2$ only model ]{
    \includegraphics[width = 0.18 \linewidth]{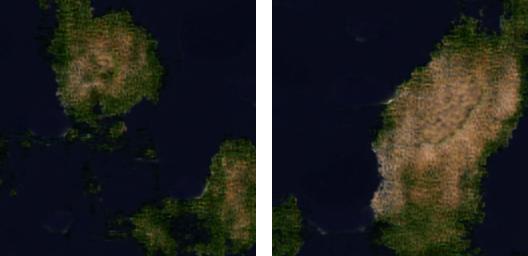}
    }
    \hfill
    \subfigure[$\infty$-GAN]{
    \includegraphics[width = 0.18 \linewidth]{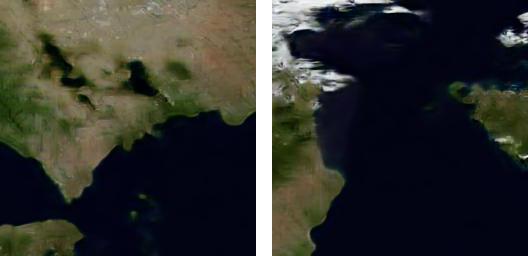}
    }
    \vspace{-5pt}
 	\caption{World map samples. The top row shows samples of size $2048 \times 2048$ (except for the data), and the bottom rows shows data/generated patches of size $256 \times 256$.}
 	\label{fig:world}
 	\vspace{-10pt}
  \end{figure*}

\section{Experiments}

To further demonstrate the theoretical results above, we use the theory to design $\infty$-GAN, which is built on Spatial GAN \citep{jetchev2016texture} and SinGAN \citep{shaham2019singan}. There are two major modifications: (1) $\infty$-GAN's generators are constructed only using the consistent and stationarity preserving transforms; (2) we slightly modify SinGAN's multi-scale architecture to improve computational efficiency by training with a constant image patch size regardless of the output resolution. Details of the $\infty$-GAN architecture are illustrated in Appendix \ref{sec:appendix_infinite_gan}.

We evaluate the $\infty$-GAN model on three datasets of large images: a satellite map of the world, a panoramic image, and texture images. We present the map generation results in the main text and refer the readers to Appendix \ref{sec:appendix_more_exp} for others. Details on data collection, network architectures and hyper-parameter tuning is presented in Appendix \ref{sec:appendix_exp_details}. Generated samples up to size 4096x4096 are provided in \href{https://drive.google.com/drive/folders/14VgV-GMNIfK7qglIUhr-Q96Sd_mnPmit?usp=sharing}{this url}.

\subsection{Evaluation on Satellite World Map}

We consider the task of generating novel and realistic looking satellite world maps. Our $\infty$-GAN model is trained on a $1150 \times 2700$ satellite image taken from the NASA visible Earth project (see Appendix \ref{sec:appendix_exp_details}).  Training patches are obtained by first down-sizing the modified world map image by $1, 2$ and $4\times$, then randomly cropping patches of size $64 \times 64$. This returns three datasets $\mathcal{D}_0, \mathcal{D}_1, \mathcal{D}_2$ containing images patches with increasing resolutions, which are used to train the multi-scale generator at different scale levels.

Here we consider three baseline models for comparison. The first baseline is a Spatial GAN, and the second basesline is a Periodic Spatial GAN (PSGAN) \citep{bergmann2017learning} which is a Spatial GAN with global latent variables and learnable periodic embeddings. Both models has model patch size $256\times 256$ and they are trained on $256\times 256$ crops from the full-scale image. The third baseline is our extension network $G_0(\cdot)$ trained on $64 \times 64$ patches from $\mathcal{D}_2$, which represents a version of our model without the multi-scale architecture. We refer to this baseline as the $\mathcal{D}_2$ only model. We compare the models quantitatively using the Frechet Inception Distance (FID) \citep{heusel2017gans} on the generated outputs with $\mathcal{D}_2$ level resolution, by generating 5 images of size $2048 \times 2048$, sampling 10,000 random crops to construct a set of ``fake image patches'', and then comparing with the random cropped patches from the full-scale data image.
The patch-level FID is computed using different crop sizes; with small patch size the score measures the quality in terms of terrain texture, and with large patch size the score evaluates the high-level structures of the generated images.

The FID scores are reported in Table \ref{tab:world}. We see that both the Spatial GAN and the PSGAN fail to capture the terrain textures from the real world map. They perform slightly better than the $\mathcal{D}_2$ only model in terms of FID scores on larger patches, which is reasonable as the $\mathcal{D}_2$ only model is trained on $64 \times 64$ patches only, therefore the $\mathcal{D}_2$ only model cannot represent the high-level structures in the original world map. 
On the other hand, the full $\infty$-GAN with fixed size multi-scale training significantly outperforms the baselines at all patch levels.

The quantitative results are supported by the visualizations in Figure \ref{fig:world}. Here the $\mathcal{D}_2$ only model fails to capture global statistics and so it only generates small islands and never generates a continent. Although the spatial GAN can generate continents, the image quality is much lower than our model.  PSGAN, on the other hand, fails to generate diverse outputs, showing clear issues of data memorization.
Compared to the baselines, the $\infty$-GAN model can generate novel world maps without just memorizing the data. The generated map contains realistic looking continents, and the image quality is significantly better than the two baselines.

\section{Discussion}

We presented a theory of infinity in generative modelling, based on which deep generative models are capable of generating unbounded images of high resolutions. According to our theory, we provided an example of $\infty$-GAN, whose generator is built with a consistent transform over spatial stochastic processes. This design is crucial for efficient incremental generation compared to inconsistent generators such as the Spatial GAN. The model is trained by a novel multi-scale learning algorithm using images of fixed size, which does not require full-scale high resolution images for learning. Experiments showed that the $\infty$-GAN can capture both the global dependencies shared across patches, as well as the details for each of the high resolution patches, which further demonstrated our theory.

Our work is related to research in generating high fidelity images \citep{karras2018progressive,karras2018style,brock2018large}, stochastic processes \citep[see e.g.][]{huang1984autoregressive,haining1978moving,ripley1981spatial,anselin2013spatial}, super-resolution \citep{ledig2017photo,sonderby2016amortised,wang2018high} and texture synthesis \citep{bergmann2017learning,jetchev2016texture}. These related works are discussed extensively in Appendix \ref{sec:appendix_related_work}.

The closest related work in the context of application is SinGAN \citep{shaham2019singan} which, after trained on a single image, can also generate images of arbitrary size at test time by changing the dimensions of the noise maps. They also conducted an empirical investigation on the artifacts introduced by padding in terms of reduced diversity, and their solution followed the CNN architecture of \citet{ioffe2015batch} and remove padding in the convolutional layer. Our theoretical analysis goes one step further to analyse popular architectural choices in building convolutional generators, and the usage of consistent and stationarity preserving transforms ensures the validity of $\infty$-GAN's generative model. Another difference is that the generators in SinGAN take as input the images of increasing size as training progresses. By contrast, we feed image patches of fixed size at all resolutions into $\infty$-GAN's generators in the hierarchy, which enables more efficient training. 

Future work will consider extending our theory to the non-stationary case. Advanced neural network models will be developed to represent non-stationary information in a spatial stochastic process.

\bibliography{sample}
\bibliographystyle{icml2020}

\clearpage
\appendix
\onecolumn

\section{Redundancy of incremental generation with inconsistent architectures}
\label{sec:app_redunndancy}

We assume an inconsistent generator $G(\cdot)$ contains $K$ blocks of $\{ \text{nearest up-sampling}, \text{conv } 3 \times 3 \}$ layers. Here the up-sampling scale is $U=2$ and the $3 \times 3$ convolution uses stride 1 and zero padding. Therefore given $\bm{z}$ input of spatial size $(N, N)$, the output spatial size is $(2^K N, 2^K N)$. Assume the computational budget allows generations of patches of size at most $S \times S$. This means $K$ is at most $\lfloor \log_2 S \rfloor$. 

We define \emph{inconsistent pixels} in an image patch as the pixels that are dependent on at least one of the zeros padded in the convolutional layers. Other pixels in that patch are \emph{consistently generated}, and in incremental generation, we only maintain consistent pixels and throw away the inconsistent ones. 

Recall the notation for the input index set as $J = \{a : b - 1 \}^2$ with $\{a: b - 1 \} = \{a, a+1, ..., b - 1 \} $. For convenience in the rest of this section we also write $[a : b) = \{a : b - 1 \} = \{a, a + 1, ..., b - 1 \}$ using python's indexing method. 

\begin{lemma}
\label{lemma:upsampling_output_index}
If the input index set is $J = [a : b )^2$, then the output of the generator $\bm{x}(I^K) = G(\bm{z}(J))$ has index set $I^K = [ 2^Ka : 2^Kb )^2$, and only the pixels in $\bm{x}(\tilde{I}^K)$ with $\tilde{I}^K = [ 2^Ka + 2^K - 1 : 2^K b - 2^K + 1 )^2$ are consistently generated.
\end{lemma}
\begin{proof}
It is clear that a nearest up-sampling layer with scale 2 maps a pixel with index $(i, j)$ to a $2 \times 2$ patch with index set $\{ 2i : 2i + 1 \} \times \{2j : 2j + 1 \}$. Also notice that convolution with zero padding does not change the spatial dimension. Therefore after $K$ blocks of transforms, the upper-left corner pixel in latent space $\bm{z}(a, a)$ gets mapped to the upper-left corner of the output image patch with the upper-left corner pixel $\bm{x}(2^Ka, 2^Ka)$. At the same time, the bottom-right corner pixel in latent space $\bm{z}(b-1, b-1)$ gets mapped to the bottom-right corner of the output image patch with the bottom-right corner pixel $\bm{x}(2^Kb - 1, 2^Kb - 1)$. 
Therefore the output index set is $I^K = [ 2^Ka: 2^Kb )^2$. 

We prove the second part of the lemma by induction. Assume at the $(k-1)^{th}$ block the output patch has index set $[\alpha : \beta )^2$, and it has inconsistent feature variables around the edge with width $w$. This means after the $k^{th}$ up-sampling, the set of inconsistent feature variables are also around the edge with width $2w$. Then after the $3 \times 3$ convolution, the output feature variables with index $i \in \{2 \alpha + 2w + 1, 2 \beta - 2w - 1\}$ and/or $j \in \{2 \alpha + 2w + 1, 2 \beta - 2w - 1\}$ are dependant on the inconsistent feature variables in the last layer thereby inconsistent as well. Therefore at the $k^{th}$ block the output patch has index set $[2 \alpha : 2 \beta )^2$, and the inconsistent feature variables are around the edge with width $2w + 1$. Since the input of the $1^{st}$ block are all consistent latent variables, the image ouput $\bm{x}(I^K)$ has $2^K - 1$ inconsistent pixels along each side of the image. Therefore the consistent pixels are $\bm{x}(\tilde{I}^K)$ with $\tilde{I}^K = [ 2^Ka + 2^K - 1: 2^K b - 2^K + 1 )^2$.
\end{proof}

In incremental generation of neighbouring patches $\bm{x}(I^K_1)$ and $\bm{x}(I^K_2)$, the corresponding index sets $J_1$ and $J_2$ in $\bm{z}$ need to have overlap. Otherwise the patches are independent to each other, which is undesirable. However, if the overlap $J_1 \cap J_2$ is too large, then many consistent pixels in the image space get wasted. So in the following lemma we calculate the overlap size which achieves the lowest redundancy of generated pixels.

\begin{lemma}
\label{lemma:overlap_z_space}
To achieve consistent incremental generation of two neighbouring patches, $J_1$ and $J_2$ only needs to overlap by 2 columns or 2 rows in $\bm{z}$ space, when $K \geq 2$ blocks are in use. In this case the consistent output patches $\bm{x}(\tilde{I}^K_1)$ and $\bm{x}(\tilde{I}^K_2)$ have overlapping pixels of 2 columns or 2 rows.
\end{lemma}
\begin{proof}
We prove this overlap result for row overlapping, column overlapping result can be proved accordingly. We assume $J_1 = [a_1 : b_1 ) \times [\alpha : \beta )$ and $J_1 = [a_2 : b_2 ) \times [\alpha : \beta )$. Assume $a_1 < a_2 < b_1 < b_2$, it is equivalent to prove $b_1 - a_2 \geq 2$ (as now the overlapping rows in $\bm{z}$ space has row indices $\{a_2, a_2 + 1, ... b_1 - 1 \}$).

From Lemma \ref{lemma:upsampling_output_index} we see that the output patches $\bm{x}(I^K_1)$ and $\bm{x}(I^K_2)$ has index sets $I^K_1 = [2^K a_1 : 2^K b_1 ) \times [2^K \alpha : 2^K \beta )$ and $I^K_2 = [2^K a_2 : 2^K b_2 ) \times [ 2^K \alpha : 2^K \beta )$. However since the inconsistent pixels need to be removed, the resulting consistent patches has index sets $\tilde{I}^K_1 = [ 2^K a_1 + 2^K - 1 : 2^K b_1 - 2^K + 1 ) \times [2^K \alpha + 2^K - 1 : 2^K \beta - 2^K + 1 )$ and $\tilde{I}^K_2 = [2^K a_2 + 2^K - 1 : 2^K b_2 - 2^K + 1 ) \times [2^K \alpha + 2^K - 1 : 2^K \beta - 2^K + 1 )$. To make the two consistent patches $\bm{x}(\tilde{I}^K_1)$ and $\bm{x}(\tilde{I}^K_2)$ overlap or at least adjacent, it requires $2^K b_1 - 2^K + 1 \geq 2^K a_2 + 2^K - 1$, which means $b_1 \geq a_2 + 2 - 1 / 2^{K-1}$. As $K \geq 1$ and $b_1, a_2 \in \mathbb{Z}$, this means $b_1 - a_2 \geq 2$ when $K \geq 2$ and $b_1 - a_2 \geq 1$ when $K=1$. 

To prove the second part of the lemma, we set $a_2 = b_1 - 2$. We also assume $b_1 - a_1 = b_2 - a_2 = \alpha - \beta = \lfloor \log_2 S \rfloor$ in order to minimize the number of inconsistent pixels to be discarded. This means the consistent output patches $\bm{x}(\tilde{I}^K_1)$ and $\bm{x}(\tilde{I}^K_2)$ have overlapping pixels, i.e.~$\tilde{I}^K_1 \cap \tilde{I}^K_2 = [2^K a_2 + 2^K - 1 : 2^K b_1 - 2^K + 1) \times [2^K \alpha + 2^K - 1 : 2^K \beta - 2^K + 1 )$. Note that $2^K b_1 - 2^K + 1 - (2^K a_2 + 2^K - 1) = 2$ since $b_1 - a_2 = 2$. Therefore the two consistent output patches have overlapping pixels of 2 rows.
\end{proof}

Lemma \ref{lemma:overlap_z_space} indicates that the two rows of consistent pixels $\bm{x}([2^K a_2 + 2^K - 1 : 2^K b_1 - 2^K + 1) \times [2^K \alpha + 2^K - 1 : 2^K \beta - 2^K + 1 ))$ is generated twice, and if $\bm{x}(I^K_2)$ is generated after $\bm{x}(I^K_1)$, then we need to discard these two rows of pixels in $\bm{x}(\tilde{I}^K_2)$. So in sum, for an output patch $\bm{x}(I^K)$ that is one of the interior patches in the final huge image (which is stitched from $M \times M$ patches), $2(2^K - 1)$ rows/columns of inconsistent pixels as well as 2 rows/columns of consistent pixels need to be discarded. As the final stitched image has $4M - 4$ boundary patches but $(M-2)^2$ interior patches, it means when $M \rightarrow +\infty$, the redundancy of pixel generation is dominated by the redundancy in generating the interior patches. Therefore we can prove Theorem \ref{thm:redundancy} presented in the main text.
\begin{proof}
As explained above it is sufficient to compute the fraction of discarded pixels in an interior patch. Recall the input $\bm{z}(J)$ has spatial size $(N, N)$. Then in $\bm{x}(I^K)$ which has spatial size $(2^K N, 2^K N)$, $2^{K+1}$ rows and $2^{K+1}$ columns of pixels are discarded. Therefore the percentage of discarded pixels is
$$ 1 - \frac{(2^K N - 2^{K+1})^2}{2^{2K} N^2} = 1 - \frac{(N - 2)^2}{N^2} = \frac{4}{N} - \frac{4}{N^2}. $$
Now recall that $2^K N \leq S$, meaning that $N \leq \lfloor \frac{S}{2^K} \rfloor$. This means the pixel generation redundancy is of percentage at least $4 [ (\lfloor \frac{S}{2^K} \rfloor )^{-1} - (\lfloor \frac{S}{2^K} \rfloor )^{-2} ]$.
\end{proof}

One can reduce the number of redundantly generated pixels by decreasing $K$, the number of $\{ \text{nearest up-sampling}, \text{conv } 3 \times 3 \}$ blocks. However when $K$ is small the network will fail to capture the global information presented in a big image patch. On the other hand, with large $K$, although the generator gets very flexible, in incremental generation this flexibility is severely damaged as a large fraction of pixels in patches are removed. As the overlap in $\bm{z}$ space is 2, this also means $N \geq 3$ and the number of blocks $K$ is restricted to have $K \leq \lceil \log_2 \frac{S}{3} \rceil$. Even in this case the fraction of removed pixels, when $M \rightarrow +\infty$, goes to $8/9$ which is a very significant value.


\section{Proofs for consistency and stationarity preserving results}

\subsection{Proofs for consistency}
\label{sec:consistency_proofs}

We wish to establish the consistency results for operators defined by a convolutional neural network. It it sufficient to prove that each component of this convolutional neural network, including convolution, pixel-wise non-linearity and up-sampling (bilinear or nearest interpolation), is a valid operator on stochastic processes, i.e.~the output of the operation is also a stochastic process. We follow \citet{ma2018variational} to construct our proofs, before that we explain in below the proof ideas of the consistency results. The main techniques are the Kolmogorov extension theorem and the Karhunen-Loeve expansion, and we use either of which is more convenient than the other. 

We assume w.l.o.g.~the input stochastic process as $\bm{z}(\cdot, \cdot) \sim SP(0, \mathcal{K})$ a centered discrete stochastic process on $\mathcal{L}^2(\mathbb{Z}^2)$, i.e.~$C= 1$ and $\bm{z}(i, j) \in \mathbb{R}$. 
Define the collection of random variables $\bm{x}(\cdot, \cdot) = \{ \bm{x}(i, j) = \mathcal{O}(\bm{z}(\cdot, \cdot))(i, j) \ | \ (i, j) \in \mathbb{Z}^2 \}$ as the outcome of a given operator $\mathcal{O}$. We say the operator $\mathcal{O}$ is \emph{consistent} if $\bm{x}(\cdot, \cdot)$ is also a stochastic process on $\mathcal{L}^2(\mathbb{Z}^2)$.

Also denote the distribution of a finite subset $\bm{x}(I) = \{ \bm{x}(i, j) \ | \ (i, j) \in I \} \subset \bm{x}(\cdot, \cdot)$ with index set $I \subset \mathbb{Z}^2$ as $P_{X, I}(\bm{s})$. The Kolmogorov extension theorem states the consistency conditions of a stochastic process as:
\begin{itemize}
    \item[1.] Marginalization consistency: for any finite subset $I \subset \mathbb{Z}^2$, and any $I' \subset I$, $$P_{X, I'}(\bm{x}(I')) = P_{X, I}(\bm{x}(I')) := \int P_{X, I}(\bm{x}(I)) d\bm{x}(I \backslash I').$$
    \item[2.] Permutation invance: for any permutation $I''$ of $I$, $$P_{X, I''}(\bm{x}(I'')) = P_{X, I}(\bm{x}(I)) .$$
\end{itemize}
Therefore given an operator $\mathcal{O}$, one can check the two conditions for $\bm{x}(\cdot, \cdot)$ to validate whether the output is a stochastic process. When $\bm{x}(\cdot, \cdot)$ is a stochastic process, we may omit the subscript $I$ and write the distribution as $P_{X}(\bm{x}(I))$.

One can also use the he Karhunen-Loeve expansion (K-L expansion) theorem to prove the existence of a stochastic process. This input stochastic process can be K-L expanded as a stochastic infinite series
\begin{equation}
\bm{z}(i, j) = \sum_{m}^{\infty} \sqrt{\lambda_m} Z_m \phi_{m}(i, j), \quad \sum_{m}^{\infty} \lambda_m < + \infty,
\label{eq:kl_expansion}
\end{equation}
where $\{ Z_m \}$ is a collection of zero-mean, unit-variance, uncorrelated variables, and $\{ \phi_m(\cdot, \cdot) \}$ is an orthonomal basis of $\mathcal{L}^2(\mathbb{Z}^2)$.
Our proofs use the K-L expansion of $\bm{z}(\cdot, \cdot)$ and establish the convergence of the operator applied on the stochastic infinite series in $\mathcal{L}^2(\mathbb{Z}^2)$. Since in this case $\bm{x}(\cdot, \cdot)$ can also be represented as a convergent stochastic infinite series, it means that $\bm{x}(\cdot, \cdot)$ is also a stochastic process.

In the following lemmas we consider the operator as convolution, pixel-wise non-linear transformation and bi-linear up-sampling.

\begin{lemma}
\label{lemma:conv_consistency}
The convolution operator with finite filter size $K$, stride 1 and \textbf{zero/constant padding} is inconsistent.
\end{lemma}
\begin{proof}
The convolution operator $\mathcal{O}$ transforms a $H \times W \times C$ image patch tensor to another $H \times W \times C$ image patch by padding in zeros or constants around the edge of the input patch. This violates the marginalization consistency requirement in Kolmogorov extension theorem. By taking $I = \{1: H+1\} \times \{ 1: W+1\}$ and $I' =  \{1: H\} \times \{ 1: W\}$ and using the input process $\bm{z}(\cdot, \cdot)$ as i.i.d.~Gaussian with non-zero mean and marginal variance $\sigma^2 \rightarrow 0$, it is straight-forward to see the equality in the marginalization consistency requirement does not hold.
\end{proof}

\begin{lemma}
For a convolution operator with finite filter size $K$, stride 1 and \textbf{no padding}, $\bm{x}(\cdot, \cdot)$ is a stochastic process on $\mathcal{L}^2(\mathbb{Z}^2)$ if $\bm{z}(\cdot, \cdot)$ is a stochastic process on $\mathcal{L}^2(\mathbb{Z}^2)$.
\end{lemma}

\begin{proof}
Here we consider w.l.o.g.~a convolutional filter with filter size $K = 2K'+1$ output channel 1. This means the parameters of the convolutional filter can be represented as $F \in \mathbb{R}^{K \times K \times C}$, $F = \{ F(k, k') \in \mathbb{R}^C \ | \ (k, k') \in \{-K': K' \}^2 \}$ and the convolution operation is defined as
$$\bm{x}(i, j) = \mathcal{O}(\bm{z}(\cdot, \cdot))(i, j) = \sum_{k=-K'}^{K'} \sum_{k'=-K'}^{K'} \langle F(k, k'), \bm{z}(i - k, j - k') \rangle.$$
It is straight-forward to show that 
$$\tilde{\bm{x}}(i, j) = \sum_{k=-K'}^{K'} \sum_{k'=-K'}^{K'} |\langle F(k, k'), \bm{z}(i - k, j - k') \rangle| < \infty$$
as $F(k, k') \in \mathbb{R}^C$, $\bm{z}(i, j) \in \mathbb{R}^C$, a dot product in $\mathbb{R}^C$ is finite, and the above equation is a finite sum. Therefore we can show via Fubini's theorem, $|\bm{x}(i, j)| \leq \tilde{\bm{x}}(i, j) < +\infty$, and
\begin{equation}
\begin{aligned}
\bm{x}(i, j) &= \sum_{k=-K'}^{K'} \sum_{k'=-K'}^{K'} \langle F(k, k'), \sum_{m}^{\infty} \sqrt{\lambda_m} Z_m \bm{\phi}_m(i - k, j - k') \rangle \\
&= \sum_{m}^{\infty} \sqrt{\lambda_m} Z_m \sum_{k=-K'}^{K'} \sum_{k'=-K'}^{K'} \langle F(k, k'), \bm{\phi}_m(i - k, j - k') \rangle.
\end{aligned}
\end{equation}
To show this stochastic infinite series converge in $\mathcal{L}^2(\mathbb{Z}^2)$, we need to show
$$ || \bm{x}(\cdot, \cdot) ||^{2}_{\mathcal{L}^2(\mathbb{Z}^2)} \leq || \mathcal{O}||^2 || \sum_{m}^{\infty} \sqrt{\lambda_m} Z_m \phi_m (\cdot, \cdot)  ||^{2}_{\mathcal{L}^2(\mathbb{Z}^2)} = || \mathcal{O}||^2 || \sum_{m}^{\infty} \lambda_m ||Z_m||^2_2 < +\infty.$$
Here the operator norm is defined as
$$|| \mathcal{O}|| := \inf \{ a \geq 0 \ | \ || \mathcal{O}(f) ||_{\mathcal{L}^2(\mathbb{Z}^2)} \leq a ||f||_{\mathcal{L}^2(\mathbb{Z}^2)}, \forall f \in \mathcal{L}^2(\mathbb{Z}^2) \}, $$
and it is straight-forward to show for convolution, $||\mathcal{O}|| = \max |F(k, k', c)|$, which is the the maximum absolute value of the filter tensor $F$ thus finite. Also $\sum_{m}^{\infty} \lambda_m ||Z_m||^2_2$ converges almost surely since $\sum_{m}^{\infty} \lambda_m < + \infty$.
Therefore $|| \bm{x}(\cdot, \cdot) ||^{2}_{\mathcal{L}^2(\mathbb{Z}^2)} < + \infty$, and $\bm{x}(\cdot, \cdot)$ is a stochastic process on $\mathcal{L}^2(\mathbb{Z}^2)$.
\end{proof}

\begin{lemma}
For a pixel-wise nonlinear operator defined by an activation function $f(\cdot)$, if there exists $0 \leq A < + \infty$ such that $f(x) \leq A|x|, \forall x \in \mathbb{R}$, then $\bm{x}(\cdot, \cdot)$ is a stochastic process on $\mathcal{L}^2(\mathbb{Z}^2)$ if $\bm{z}(\cdot, \cdot)$ is a stochastic process on $\mathcal{L}^2(\mathbb{Z}^2)$.
\end{lemma}
\begin{proof}
The pixel-wise nonlinear operator with activation function $f(\cdot)$ is defined as
$$\bm{x}(i, j) = \mathcal{O}(\bm{z}(\cdot, \cdot))(i, j) = f(\bm{z}(i, j)).$$
By the assumption on $f(\cdot)$ we have $|\bm{x}(i, j)| \leq A |\bm{z}(i, j)|$. Since $||\bm{z}(\cdot, \cdot) ||^2_{ \mathcal{L}^2(\mathbb{Z}^2)} < +\infty$, this also indicates that $||\bm{x}(\cdot, \cdot) ||^2_{ \mathcal{L}^2(\mathbb{Z}^2)} < +\infty$, and $\bm{x}(\cdot, \cdot)$ is a stochastic process on $\mathcal{L}^2(\mathbb{Z}^2)$.
\end{proof}
We note that commonly used activation functions (sigmoid, ReLU, etc.) satisfy the condition required by this lemma. This condition is also not a necessary condition, see the proof for the following lemma which can also be extended to commonly used activation functions.

\begin{lemma}
Pixel normalization is a consistent transform for stochastic processes.
\end{lemma}
\begin{proof}
Assume $\bm{z}(i, j) \in \mathbb{R}^C$ and each of the values in a channel is $\bm{z}(i, j, c)$. Pixel normalization is defined as the following:
$$\bm{x}(i, j) = \text{PixelNorm}(\bm{z}(i, j)) = \frac{\bm{z}(i, j) - \mu(i, j)}{\sigma(i, j)},$$
$$\mu(i, j) = \frac{1}{C} \sum_{c=1}^C \bm{z}(i, j, c), \quad \sigma^2(i, j) = \frac{1}{C} \sum_{c=1}^C (\bm{z}(i, j, c) - \mu(i, j) )^2.$$
Recall that the permutation test is performed on pixel level. As PixelNorm is performed on pixel level as well, it means the permutation invariance condition is satisfied.

We now prove the marginal consistency condition for pixel normalization. We can define the inverse of the operator $\mathcal{O}$ induced by PixelNorm, and since pixel normalization is performed on pixel level, it is easy to show that $\mathcal{O}^{-1}(\bm{x}(I) = \bm{b}) = \times_{i, j \in I} \mathcal{O}^{-1}(\bm{x}(i, j) = \bm{b}(i, j))$, with $\mathcal{O}^{-1}(\bm{x}(i, j) = \bm{b}(i, j)) = \{ \bm{z}(i, j) = \alpha \bm{x}(i, j) + \beta, \forall \alpha, \beta \in \mathbb{R} \}.$ This implies for any finite subsets $I' \subset I$ and the corresponding $\bm{b}' = \bm{b}(I')$, we have
\begin{equation*}
\begin{aligned}
P_{X, I}(\bm{x}(I') = \bm{b}') &= \int \int_{\mathcal{O}^{-1}(\bm{x}(I') = \bm{b}') \times \mathcal{O}^{-1}(\bm{x}(I \backslash I') = \bm{a})} P_{Z, I}(\bm{z}(I'), \bm{z}(I \backslash I')) \ d \bm{z}(I) d \bm{a} \\
&= \int \int_{\mathcal{O}^{-1}(\bm{x}(I') = \bm{b}') \times \mathcal{O}^{-1}(\bm{x}(I \backslash I') = \bm{a})} P_{Z, I}(\bm{z}(I'), \bm{z}(I \backslash I')) \ d \bm{a} d \bm{z}(I) \\
&= \int_{\mathcal{O}^{-1}(\bm{x}(I') = \bm{b}')} P_{Z, I}(\bm{z}(I'), \bm{z}(I \backslash I')) \ d \bm{z}(I) \\
&= \int_{\mathcal{O}^{-1}(\bm{x}(I') = \bm{b}')} P_{Z, I}(\bm{z}(I')) \ d \bm{z}(I') \\
&= P_{X, I'}(\bm{x}(I') = \bm{b}'),
\end{aligned}
\end{equation*}
where in the second equality the exchange of integration follows Fubini's theorem, and the fourth equality follows from the assumption that $\bm{z}(\cdot, \cdot)$ is a stochastic process. Therefore we have proved the marginal consistency condition for $\bm{x}(\cdot, \cdot)$, and $\bm{x}(\cdot, \cdot)$ is a stochastic process.
\end{proof}

\begin{lemma}
\label{lemma:bilinear_consistency}
For a bi-linear up-sampling operator with scale $U=2$ and boundary cropping, $\bm{x}(\cdot, \cdot)$ is a stochastic process on $\mathcal{L}^2(\mathbb{Z}^2)$ if $\bm{z}(\cdot, \cdot)$ is a stochastic process on $\mathcal{L}^2(\mathbb{Z}^2)$.
\end{lemma}
\begin{proof}
The bi-linear upsampling operator with scale U=2 and boundary cropping is given as follows:
\begin{equation*}
\begin{aligned}
\bm{x}(2h, 2w) &= \frac{3}{4} \left(\frac{3}{4} \bm{z}(h, w) + \frac{1}{4} \bm{z}(h, w+1) \right) + \frac{1}{4} \left(\frac{3}{4} \bm{z}(h+1, w) + \frac{1}{4} \bm{z}(h+1, w+1) \right), \\
\bm{x}(2h, 2w+1) &= \frac{3}{4} \left(\frac{1}{4} \bm{z}(h, w) + \frac{3}{4} \bm{z}(h, w+1) \right) + \frac{1}{4} \left(\frac{1}{4} \bm{z}(h+1, w) + \frac{3}{4} \bm{z}(h+1, w+1) \right), \\
\bm{x}(2h+1, 2w) &= \frac{1}{4} \left(\frac{3}{4} \bm{z}(h, w) + \frac{1}{4} \bm{z}(h, w+1) \right) + \frac{3}{4} \left(\frac{3}{4} \bm{z}(h+1, w) + \frac{1}{4} \bm{z}(h+1, w+1) \right), \\
\bm{x}(2h+1, 2w+1) &= \frac{1}{4} \left(\frac{1}{4} \bm{z}(h, w) + \frac{3}{4} \bm{z}(h, w+1) \right) + \frac{3}{4} \left(\frac{1}{4} \bm{z}(h+1, w) + \frac{3}{4} \bm{z}(h+1, w+1) \right). \\
\end{aligned}
\end{equation*}
Using Lemma \ref{lemma:conv_consistency} we can show that $\{\bm{x}(2h, 2w) \ | \ (h, w) \in \mathbb{Z}^2 \}$ is a stochastic process, and similarly the other three are stochastic processes as well. Also for any finite numbers $H, W$, if we write 
$$S(H, W, 0, 0) =  \sum_{h=-H}^W \sum_{w=-W}^W \bm{x}(2h, 2w)^2, \quad S(H, W, 0, 1) =  \sum_{h=-H}^W \sum_{w=-W}^W \bm{x}(2h, 2w+1)^2,$$
$$S(H, W, 1, 0) =  \sum_{h=-H}^W \sum_{w=-W}^W \bm{x}(2h+1, 2w)^2, \quad S(H, W, 1, 1) =  \sum_{h=-H}^W \sum_{w=-W}^W \bm{x}(2h+1, 2w+1)^2,$$
then we have
\begin{equation*}
\begin{aligned}
\sqrt{\sum_{i=-2H}^{2H+1} \sum_{j=-2W}^{2W+1} \bm{x}(i, j)^2} \leq \sum_{i=0}^1 \sum_{j=0}^1 S(H, W, i, j) 
\leq 4 \max_{i, j} S(H, W, i, j).
\end{aligned}
\end{equation*}
Also $S(H, W, i, j)$ is non-decreasing as $H, W$ increase, so $S(H, W, 0, 0) \uparrow || \bm{x}(2h, 2w) ||^2_{\mathcal{L}^2(\mathbb{Z}^2)}$ (similarly for the others). Therefore we have all $S(H, W, i, j), \forall i, j$ bounded by
\begin{equation*}
\begin{aligned}
\max_{i, j} S(H, W, i, j) \leq \max ( & || \bm{x}(2h, 2w) ||^2_{\mathcal{L}^2(\mathbb{Z}^2)}, || \bm{x}(2h, 2w+1) ||^2_{\mathcal{L}^2(\mathbb{Z}^2)}
, \\
& || \bm{x}(2h+1, 2w) ||^2_{\mathcal{L}^2(\mathbb{Z}^2)}, || \bm{x}(2h+1, 2w+1) ||^2_{\mathcal{L}^2(\mathbb{Z}^2)} ) < + \infty,
\end{aligned}
\end{equation*}
where the last inequality comes from the fact that $\{\bm{x}(2h, 2w) \ | \ (h, w) \in \mathbb{Z}^2 \}$ and the other three collections are stochastic processes on $\mathcal{L}^2(\mathbb{Z}^2)$.
Therefore the supremum in below is also finite:
\begin{equation*}
\begin{aligned}
\sup_{H, W \rightarrow +\infty} \sqrt{ \sum_{i=-2H}^{2H+1} \sum_{j=-2W}^{2W+1} \bm{x}(i, j)^2} \leq 4 \max ( & || \bm{x}(2h, 2w) ||^2_{\mathcal{L}^2(\mathbb{Z}^2)}, || \bm{x}(2h, 2w+1) ||^2_{\mathcal{L}^2(\mathbb{Z}^2)}
, \\
& || \bm{x}(2h+1, 2w) ||^2_{\mathcal{L}^2(\mathbb{Z}^2)}, || \bm{x}(2h+1, 2w+1) ||^2_{\mathcal{L}^2(\mathbb{Z}^2)} ).
\end{aligned}
\end{equation*}
Also we know the sum inside the square root on the LHS is non-decreasing when $H, W$ increase. Therefore by the monotonic convergence theorem, the limit of the sum exists and 
$$|| \bm{x}(\cdot, \cdot)||_{\mathcal{L}^2(\mathbb{Z}^2)} = \lim_{H, W \rightarrow +\infty} \sqrt{ \sum_{i=-2H}^{2H+1} \sum_{j=-2W}^{2W+1} \bm{x}(i, j)^2} < + \infty. $$
Therefore $\bm{x}(\cdot, \cdot)$ is a stochastic process on $\mathcal{L}^2(\mathbb{Z}^2)$.
\end{proof}

\begin{lemma}
\label{lemma:nearest_consistency}
For a nearest up-sampling operator with finite scale $U$, $\bm{x}(\cdot, \cdot)$ is a stochastic process on $\mathcal{L}^2(\mathbb{Z}^2)$ if $\bm{z}(\cdot, \cdot)$ is a stochastic process on $\mathcal{L}^2(\mathbb{Z}^2)$.
\end{lemma}
\begin{proof}
The nearest up-sampling operator with finite scale $U$ is given as
$$\bm{x}(Uh + i, Uw + j) = \bm{z}(h, w), \quad \forall (h, w) \in \mathbb{Z}^2, 0 \leq i < U, 0 \leq j < U.$$
Using similar ideas in the proof of Lemma \ref{lemma:bilinear_consistency}, one can show that $\bm{x}(\cdot, \cdot)$ is a stochastic process on $\mathcal{L}^2(\mathbb{Z}^2)$.
\end{proof}


\subsection{Proofs for stationarity preserving}
\label{sec:stationarity_proofs}

In below we show the stationarity preserving properties for the transformations used by an CNN. We say an operator $\mathcal{O}$ is \emph{stationarity preserving}, if for any strictly stationary spatial process $\bm{z}(\cdot, \cdot)$, the output process $\{ \bm{x}(i, j) = \mathcal{O}(\bm{z}(\cdot, \cdot))(i, j) \}$ is strictly spatial stationary. Note that it suffices to investigate stationarity properties on patches derived from $\bm{x}(\cdot, \cdot)$ as any finite subset of the index set $\mathbb{Z}^2$ is located within some finite-size patch. In the rest of the section we denote the distribution of any finite-size patch 

\begin{lemma}
\label{lemma:conv_stationarity}
Convolution with finite filter size $k$, stride 1 and no padding preserves strict spatial stationarity.
\end{lemma}
\begin{proof}
Again we consider w.l.o.g.~a convolutional filter with filter size $K = 2K'+1$ output channel 1. 
For any finite subset $I = \{1:H \} \times \{1:W \} \subset \mathbb{Z}^2$ we can ``push'' the boundary of $I$ by size $K'$ and obtain another index set $J$ so that applying the convolution filter on $\bm{z}(J)$ returns $\bm{x}(I)$. Also define the inverse set of the operator on $\bm{x}(I)$ as $$\mathcal{O}^{-1}(\bm{x}(I) = \bm{b}) = \{ \bm{a} \in \mathbb{R}^{|J| \times C} \ | \ \bm{x}(I) = \mathcal{O}(\bm{z}(\cdot, \cdot) )(I) = \bm{b}, {\bm{z}(J) = \bm{a}} \}.$$
Similarly we can define such subset $\tilde{J}$ for the shifted set $\tilde{I} = \{(i + h, j + w) \ | \ (i, j) \in I \}$ given some shifting direction $(h, w)$. Importantly, due to translation invarance of convolutions, $\tilde{J} = \{(i + h, j + w) \ | \ (i, j) \in J \}$ is also a shifted set of $J$ with shifting direction $(h, w)$. The shift invarance of convolution also means $\mathcal{O}^{-1}(\bm{x}(I) = \bm{b}) = \mathcal{O}^{-1}(\bm{x}(\tilde{I}) = \bm{b})$. Furthermore the input process is stationary, i.e.~$P_{Z, J}(\bm{z}(J)) = P_{Z, \tilde{J}}(\bm{z}(\tilde{J}))$. Putting these results together, we have
\begin{equation*}
\begin{aligned}
P_{X, I}(\bm{x}(I) = \bm{b}) &=  \int_{\mathcal{O}^{-1}(\bm{x}(I) = \bm{b} )} P_{Z, J}(\bm{z}(J)) d \bm{z}(J) \\
&= \int_{\mathcal{O}^{-1}(\bm{x}(\tilde{I}) = \bm{b} )} P_{Z, \tilde{J}}(\bm{z}(\tilde{J})) d \bm{z}(\tilde{J}) = P_{X, \tilde{I}}(\bm{x}(\tilde{I}) = \bm{b}).
\end{aligned}
\end{equation*}
Therefore $\bm{x}(\cdot, \cdot)$ is a strictly spatial stationary process.
\end{proof}

\begin{lemma}
Pixel normalization, and pixel-wise non-linear operator defined by an activation function $f(\cdot)$, preserves strict spatial stationarity.
\end{lemma}
\begin{proof}
We prove this lemma for pixel-wise non-linear operators first, the proof for Pixel normalization can be done accordingly. Since the non-linearity $f(\cdot)$ is applied pixel-wise, this means
$$\mathcal{O}^{-1}(\bm{x}(I) = \bm{b}) = \times_{(i, j) \in I} f^{-1}(\bm{x}(i, j) = \bm{b}(i, j)) = \times_{(i, j) \in I} \{ \bm{a} \in \mathbb{R}^C \ | \ f(\bm{a}) = \bm{b}(i, j) \}. $$
Also it is straight-forward that $\mathcal{O}^{-1}(\bm{x}(I) = \bm{b}) = \mathcal{O}^{-1}(\bm{x}(\tilde{I}) = \bm{b})$, and by assumption we have $P_{Z, J}(\bm{z}(J)) = P_{Z, \tilde{J}}(\bm{z}(\tilde{J}))$. Therefore one can show in similar way as to prove Lemma \ref{lemma:conv_stationarity} (with $J=I, \tilde{J} = \tilde{I}$) that $\bm{x}(\cdot, \cdot)$ is a strictly spatial stationary process.
\end{proof}

\begin{lemma}
\label{lemma:bilinear_stationarity}
Bi-linear upsampling with scale $U=2$ and boundary cropping transforms a strictly stationary spatial process to a cyclostationary spatial process of period $(2, 2)$.
\end{lemma}
\begin{proof}
It suffice to prove that $P_{X, I}(\bm{x}(I) = \bm{b}) = P_{X, \tilde{I}}(\bm{x}(\tilde{I}) = \bm{b})$ for all
$$
I = \{a, a + h \} \times \{b : b + w \}, \tilde{I} =  \{a + 2, a + 2 + h \} \times \{b + 2 : b + 2 + w \}, \forall a, b \in \mathbb{Z}, h, w \in \mathbb{N}^+.
$$
Since for any finite $a, h$ there exist finite $H_{min} < H_{max}$ such that $\{a, a + h \} \subset \{2 H_{min} : 2 H_{max} \}$, we only need to show the equality for 
$$I = \{ 2H_{min} : 2H_{max}\} \times \{ 2W_{min} : 2W_{max}\}, \forall H_{min} < H_{max}, W_{min} < W_{max} \in \mathbb{Z},$$ 
and the corresponding shifted index set is $\tilde{I} = \{ 2H_{min}+2 : 2H_{max}+2\} \times \{ 2W_{min}+2 : 2W_{max}+2\}$.
From the operator presented in the proof of Lemma \ref{lemma:bilinear_consistency} we see that $J = \{H_{min} : H_{max}+1 \} \times \{W_{min} : W_{max}+1 \}$, $\tilde{J} = \{H_{min}+1 : H_{max}+2 \} \times \{W_{min}+1, ..., W_{max}+2 \}$, and $\mathcal{O}^{-1}(\bm{x}(I) = \bm{b}) = \mathcal{O}^{-1}(\bm{x}(\tilde{I}) = \bm{b})$. Therefore the equality holds given that $\bm{z}(\cdot, \cdot)$ a strictly stationary process (using proof techniques presented in the proof of Lemma \ref{lemma:conv_stationarity}).
\end{proof}

\begin{lemma}
\label{lemma:nearest_stationarity}
Nearest upsampling with scale $U=2$ transforms a strictly stationary spatial process to a cyclostationary spatial process of period $(2, 2)$.
\end{lemma}
\begin{proof}
Similar to the proof of Lemma \ref{lemma:bilinear_stationarity}, it suffice to show $P_{X, I}(\bm{x}(I) = \bm{b}) = P_{X, \tilde{I}}(\bm{x}(\tilde{I}) = \bm{b})$ for
$I = \{ 2H_{min} : 2H_{max}\} \times \{ 2W_{min} : 2W_{max}\}, \forall H_{min} < H_{max}, W_{min} < W_{max} \in \mathbb{Z},$
and the corresponding shifted index set is $\tilde{I} = \{ 2H_{min}+2 : 2H_{max}+2\} \times \{ 2W_{min}+2 : 2W_{max}+2\}$.
From the operator presented in the proof of Lemma \ref{lemma:nearest_consistency} we see that $J = \{H_{min} : H_{max}\} \times \{W_{min} : W_{max} \}$, $\tilde{J} = \{H_{min}+1 : H_{max}+1 \} \times \{W_{min}+1, ..., W_{max}+1 \}$, and $\mathcal{O}^{-1}(\bm{x}(I) = \bm{b}) = \mathcal{O}^{-1}(\bm{x}(\tilde{I}) = \bm{b})$. Therefore the equality holds given that $\bm{z}(\cdot, \cdot)$ is a strictly stationary process (using proof techniques presented in the proof of Lemma \ref{lemma:conv_stationarity}).
\end{proof}

We note that similar proof techniques of Lemmas \ref{lemma:bilinear_stationarity} and \ref{lemma:nearest_stationarity} can be applied to show the general result: these two consistent up-sampling methods increase the stationarity period by 2.

\section{Related work: extended discussions}
\label{sec:appendix_related_work}

\paragraph{Generating high fidelity images}
State-of-the-art \emph{un-conditional} GANs are able to generate high fidelity images of size $1024 \times 1024$  \citep{karras2018progressive,karras2018style,brock2018large}. However these GAN models require full-scale images for training, and they can only generate images of the same \emph{fixed} size as the training images. By contrast, $\infty$-GAN only takes small patches of different resolutions as training data, and the generator, once trained, can generate images of arbitrary size either in one shot or incrementally.

Other non-adversarial generative models have also been shown to be capable of high fidelity image generation \citep{kingma2018glow,huang2018introvae,menick2018generating}. Specifically, the Subscale Pixel Network (SPN) \citep{menick2018generating}, as the state-of-the-art auto-regressive generative model, uses a multi-scale ordering by sub-sampling pixel locations into rows and columns and progressing in raster ordering \citep{reed2017parallel}. This particular ordering requires a canvas of pre-defined size, making SPNs incapable of generating images of arbitrary/infinite sizes. 

\paragraph{(Stationary) stochastic processes}
Auto-regressive processes, moving average processes and their combinations have been widely applied to spatial data modelling \citep[see e.g.][]{huang1984autoregressive,haining1978moving,ripley1981spatial,anselin2013spatial}. The $\infty$-GAN's generator can be viewed as an extension to the non-linear moving average model \citep{robinson1977estimation}; it is a consistent transform of stochastic processes, and the dependence between an image patch and the corresponding latent variables is independent with the location of the patch in that image.

\paragraph{Super resolution}
GANs have also been applied to super resolution, i.e.~estimating a high resolution image from its low-resolution counterpart \citep{ledig2017photo,sonderby2016amortised,wang2018high}. Our up-scaling network is similar to StackGAN++ \citep{zhang2018stackgan++}: both methods train a stack of generators (with the pre-image outputs from the last generator as the inputs) on multi-resolution images \citep{denton2015deep, zhang2017stackgan}. Still all existing methods including StackGAN++ require full images for training, while our up-scaling network is trained on patches. 

\paragraph{Texture synthesis}  
Give a reference texture image, model-free texture synthesis techniques perform \emph{conditional} generation by re-sampling pixels or patches from the original texture \citep[see e.g.][]{barnes2009patchmatch,efros1999texture,wei2000fast,efros2001image,kwatra2003graphcut}. Perhaps PatchMatch \citep{barnes2009patchmatch} is the closest related approach, which fills in the missing values by a fast nearest neighbour search over patches from a reference image.

Model-based approaches that use CNNs have recently become popular for texture synthesis. With a pre-trained network as texture feature extractor \citep{gatys2015neural,gatys2015texture,johnson2016perceptual}, recent approaches used (conditional) GANs to generate texture images with features that are similar to the features from either a reference input \citep{ulyanov2016texture,li2016precomputed}, or from a pre-defined set known as style bank \citep{ulyanov2017improved}. The closely related approach is the (Periodic) Spatial GAN (PSGAN) \citep{bergmann2017learning,jetchev2016texture}, however it uses an inconsistent generator, and it requires large images for learning the stationarity pattern. By contrast, our model only sees small patches, and the dependency between patches is induced by the overlapping pixels in $\bm{z}$ space.

\paragraph{Terrain/map generation}
Traditionally terrain/map generation for video games are dominated by procedural generation techniques, e.g.~procedural noise functions  \citep{perlin1985noise,perlin2002improving,fournier1982computer} and physical process simulations \citep{musgrave1989synthesis,olsen2004realtime}. A recent attempt to apply GANs to world map generation is presented in \citep{beckham2017step}, which first trained a DCGAN \citep{radford2015unsupervised} on $512 \times 512$ crops from the NASA height map data of the Earth, then applied pix2pix \citep{isola2017image} to paint the terrain texture conditioned on the height map. This approach uses vector inputs and inconsistent CNNs, thus it can only generate terrains of size $512 \times 512$.

\begin{figure}[t]
\centering
\includegraphics[width=1\linewidth]{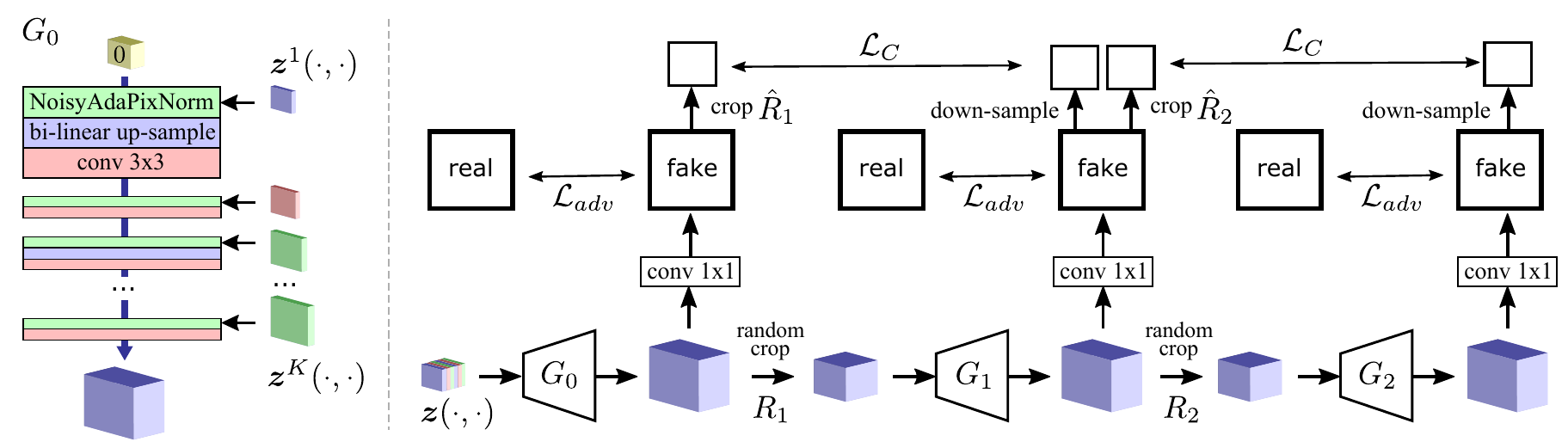}
    \caption{The network architecture ($L=2$) and the training mechanism. }
    \label{fig:architecture}
\end{figure}

\section{Infinite GAN}
\label{sec:appendix_infinite_gan}

As motivated by the theoretical results above, we design $\infty$-GAN, whose generator is well-defined in terms of our theory, to allow generation of images at different \emph{resolutions} and in arbitrary image \emph{sizes}.\footnote{Note the difference between image resolution and image size.}
To enable efficient training on vary large images, we follow the multi-scale training strategy \citep{denton2015deep,zhang2017stackgan,zhang2018stackgan++,karras2018progressive} to construct a stack of $L+1$ generators $\{G_0(\cdot), ..., G_{L}(\cdot) \}$, each of them transforms a spatial stochastic process to another spatial stochastic process.
Practically, this means given a sample of the latent tensor $\bm{z}(J)$, $\infty$-GAN is capable of generating a set of images $\{\bm{x}_0(I_0), ..., \bm{x}_L(I_L) \}$ of different resolutions. Here $I_l$ is the corresponding pixel index set of the $l^{th}$ generator's image output. 
We name the first generator $G_0(\cdot)$ as the \emph{extension network} of $\infty$-GAN's generators, which is mainly responsible for learning the dependency structure in data that is crucial for incrementally extending images to larger \emph{sizes}. As the other networks $G_l(\cdot), l \geq 1$ in the hierarchy are used for up-scaling the \emph{resolution} of the images, these networks are referred to as the \emph{up-scaling networks}. 

\paragraph{The extension network}
The extension network $G_0(\cdot)$ takes in as inputs finite size samples from multiple latent processes $\{\bm{z}^1(\cdot, \cdot), ..., \bm{z}^K(\cdot, \cdot) \}$, and produces a \emph{pre-image feature tensor} $\bm{f}_0(I_0)$ that is later transformed into an image patch $\bm{x}_0(I_0)$ with $1 \times 1$ convolution \citep{zhang2018stackgan++,karras2018progressive}. See Figure \ref{fig:architecture} for a visualization, in detail, we construct $G_0(\cdot)$ with $K$ blocks of $\{\text{Noisy AdaPixNorm, Conv } 3 \times 3 \}$ layers (possibly with up-sampling layers in the blocks), and the convolutional layers use stride 1 and no padding. The Noisy AdaPixNorm replaces instance normalization \citep{ulyanov2016instance} in StyleGAN's noise-in AdaIN layers \citep{karras2018style} with pixel normalization.\footnote{Even when the input stochastic process is stationary and ergodic, instance norm's estimation of ensemble mean and variance across pixel indices can vary a lot depending on the image size.} 
Therefore with the inputs $\{ \mathbf{h}(i, j) \in \mathbb{R}^C \}$ and the latent tensor values $\{ \bm{z}^k(i, j) \in \mathbb{R} \}$, the $k^{th}$ Noisy AdaPixNorm layer's output is $\mathbf{y}(i, j) = \bm{\beta}_k \odot \text{PixelNorm}(\mathbf{h}(i, j) + \bm{z}^k(i, j) \bm{w}_k ) + \bm{\gamma}_k$, where $\bm{\beta}_k, \bm{\gamma}_k, \bm{w}_k \in \mathbb{R}^C$.
For training, we apply the Wasserstein GAN approach with gradient penalty (WGAN-GP)  \citep{gulrajani2017improved}, using real image \emph{patches} sampled from the data $\mathcal{D}_0$ and the generator:
\begin{align}
\mathcal{L}_{adv}(G_0, D_0) = 
\mathbb{E}_{\mathbf{x} \sim \mathcal{D}_0}[D_0(\mathbf{x})] - \mathbb{E}_{\mathbf{x} \sim P_{X_0}}[D_0(\mathbf{x})] + \lambda \mathbb{E}_{\mathbf{x} \sim \tilde{P}_{X_0}}[( ||\nabla_{\mathbf{x}} D_0(\mathbf{x}) || - 1)^2].
\label{eq:wgan_gp_loss}
\end{align}
Here the image patches in $\mathcal{D}_0$ are generated by random cropping from a very large training image.
The patch size of crops needs to be selected carefully. Assuming the extension network contains $K$ up-sampling layers with scale 2, this means the \emph{model patch}, defined as the projected field in $\bm{x}_0$ space for a pixel in $\bm{z}^1$ space, has size $(2^K, 2^K)$. Also due to the usage of consistent operations, it requires the latent tensor in $\bm{z}^1$ space to have spatial size at least $(h, w)$ in order to generate a model patch in $\bm{x}_0$ space. Assuming i.i.d.~Gaussian variables for all latent processes, this means the output process $\bm{x}_0(\cdot, \cdot)$ has stationarity period $(2^K, 2^K)$, and it is $(2^Kh, 2^Kw)$-dependent in the sense that any two pixels $\bm{x}_0(i_1, j_1)$ and $\bm{x}_0(i_2, j_2)$ with $|i_1 - i_2| > 2^Kh$ and/or $|j_1 - j_2| > 2^Kw$ are independent. This means one can select the patch size $(H, W)$ of the training images as $(2^{K+1}, 2^{K+1}) \leq (H, W) \leq (2^K h, 2^K w)$ in order to learn both the marginal distribution over a model patch, as well as the dependencies between model patches. Empirically we select $(H, W) = (2^{K+1}, 2^{K+1})$, which means in training the spatial size of the latent tensor in $\bm{z}^1$ space is $(h+1, w+1)$.

\paragraph{The up-scaling networks}
Each up-scaling network $G_l(\cdot)$ takes the pre-image output $\bm{f}_{l-1}(J_l)$ from the last generator with some index set $J_l$, and produces a feature tensor $\bm{f}_l(I_l)$ with larger spatial size by consistent up-sampling and convolution. This feature tensor is then transformed into the $l^{th}$ layer's image output $\bm{x}_l(I_l)$ by a $1 \times 1$ convolution. 

For training, apart from the WGAN-GP loss (\ref{eq:wgan_gp_loss}), up-scaling consistency loss is applied to enforce the alignment of images in the multi-scale hierarchy.
However, with up-sampling layers in use, the image size of $\bm{x}(I_l)$ grows exponentially. We introduce an efficient training method for the up-scaling networks $G_l(\cdot), l = 1, ..., L$, by using image patches of fixed size at all resolutions in the hierarchy. Assume that all the generated images during training have fixed size $(H, W)$. Then one can perform a \emph{random crop} operation with index set $J_l = R_l(I_{l-1})$ on the feature tensor $\bm{f}_{l-1}(I_{l-1})$ to obtain the input $\bm{f}_{l-1}(J_l)$ for $G_l(\cdot)$, and the cropping operator $R_l$ is selected so that the generated image $\bm{x}_{l}(I_l)$ has the desired size $(H, W)$. Later $\bm{x}_l(I_l)$ is down-sampled (with method $d(\cdot)$) to match the resolution of $\bm{x}_{l-1}$, and we crop the corresponding low-resolution patch from $\bm{x}_{l-1}(I_{l-1})$ with index set $\hat{I}_{l-1} = \hat{R}_l(I_{l-1})$ to compute the consistency loss between images of different resolutions:
\begin{equation}
\mathcal{L}_{C}(G_l) = 
\mathbb{E}_{\bm{x}_{l-1} \sim P_{X_{l-1}}} \mathbb{E}_{\bm{x}_{l} \sim P_{X_{l}}} \left[ || \bm{x}_{l-1}(\hat{I}_{l-1})
- d(\bm{x}_l(I_l)) ||_2^2\right].
\end{equation}
Since the up-scaling network uses consistent up-sampling methods, the zoom-in ratio $|I_l| / |J_l|$ is not necessarily a power of two. Therefore we define $\hat{R}_l$ such that $\bm{x}_{l-1}(\hat{I}_{l-1})$ is the centered $(H/2, W/2)$ patch within $\bm{x}_{l-1}(J_l)$. 
In sum, the total loss for training the up-scaling network is:
\begin{equation}
\mathcal{L}(G_l, D_l) = \mathcal{L}_{adv}(G_l, D_l) + \lambda_{l} \mathcal{L}_{C}(G_l).
\end{equation}
See Figure \ref{fig:architecture} for a visualization. We further add a consistency loss between images generated by $G_0$ and down-sampled images from $G_L$, which we empirically find to improve results.
The images patches in $\mathcal{D}_l$ are again generated by random crops from real images, and for all $l \geq 0$ the image patches have the same fixed size. It also means the discriminator $D_l$ used in $\mathcal{L}_{adv}$ sees small patches only, unlike previous multi-scale training methods \citep{zhang2018stackgan++,karras2018progressive,denton2015deep} where the discriminators in the hierarchy observe images of exponentially increasing size. 


\section{Additional experiments}
\label{sec:appendix_more_exp}
\subsection{Qualitative evaluations on texture and panoramic data}

\paragraph{Texture generation}
We train the $\infty$-GAN model on 4 texture images. The full-scale data has size $512 \times 512$ (see the first column in Figure \ref{fig:texture}),\footnote{The visualized images are compressed due to ArXiv file size limits. For the original images see \href{https://drive.google.com/drive/folders/14VgV-GMNIfK7qglIUhr-Q96Sd_mnPmit?usp=sharing}{this url}.} and the three datasets $\mathcal{D}_0, \mathcal{D}_1$ and $\mathcal{D}_2$ are constructed in a similar way as done in the world generation task. The third column of Figure \ref{fig:texture} shows that the $\infty$-GAN has captured the texture features, even when the network has never observed $256 \times 256$ image patches with the full-scale resolution. 
We further qualitatively evaluate the $2048 \times 2048$ images generated by $\infty$-GAN (the last 4 columns); for reference we tiled the original images to the same size, which is shown in the second column of Figure \ref{fig:texture}. It is clear that the $\infty$-GAN model, after training, is capable of generating both realistic looking and non-repeating texture images.

\newcommand{\texture}[1]{
      \includegraphics[width=0.1\textwidth]{figs/textures/#1.jpg}&
    \includegraphics[width=0.1\textwidth]{figs/textures/#1-tiled.jpg}&
    \includegraphics[width=0.1\textwidth]{figs/textures/#1-fake-0_512.jpg}&
    \includegraphics[width=0.1\textwidth]{figs/textures/#1-fake-0.jpg}&
    \includegraphics[width=0.1\textwidth]{figs/textures/#1-fake-1.jpg}&
    \includegraphics[width=0.1\textwidth]{figs/textures/#1-fake-2.jpg}&
    \includegraphics[width=0.1\textwidth]{figs/textures/#1-fake-3.jpg}\\
}
\begin{figure}[t]
  \centering
  \begin{tabular}{cc||c|cccc}
    \texture{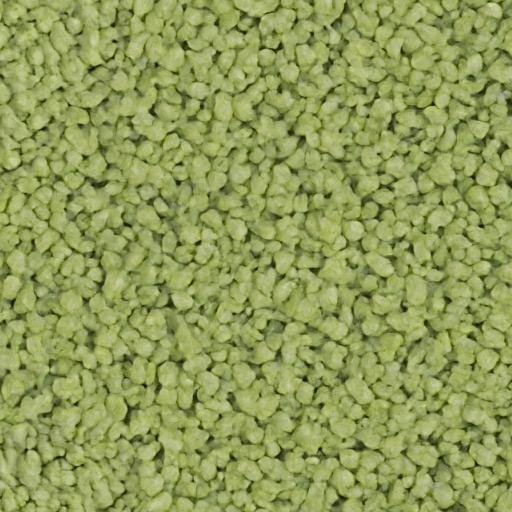}
    \texture{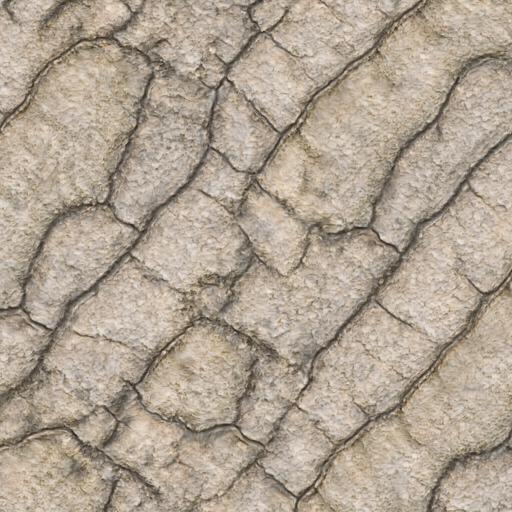}
    \texture{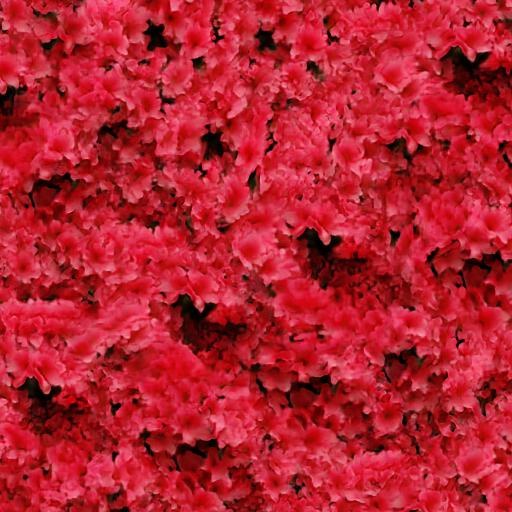}
    \texture{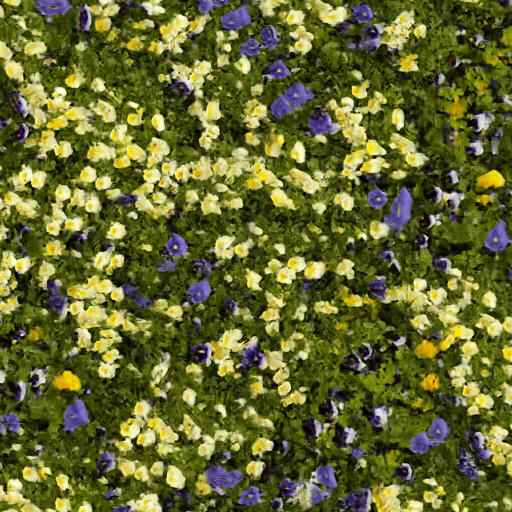}
    Data&Tiled Data&$\infty$-GAN & $\infty$-GAN & $\infty$-GAN & $\infty$-GAN & $\infty$-GAN\\
    {\small (512x512)}&{\small (2048x2048)} &{\small (512x512)}&{\small (2048x2048)}&{\small (2048x2048)}&{\small (2048x2048)}&{\small (2048x2048)}\\
    &&Sample 1&Sample 1&Sample 2&Sample 3& Sample 4\\
    \end{tabular}
  	\caption{Visualizing generated textures images. The first two columns show the training data and its tiled version. The third column shows generated images at the same scale as the input, and the last four columns show the huge samples as compared to the tiled training images. Best view in color. \label{fig:texture}}
  	\vspace{-5pt}
  \end{figure}

\paragraph{Panoramic city view}
Lastly we train the extension network $G_0(\cdot)$ of the $\infty$-GAN on a panoramic image of the New York city, and then use the model to extend the landscape view horizontally. The training data and the generated samples are shown in Figure \ref{fig:nyc_result}. We see that the $\infty$-GAN model is able to generate city views of different patterns (e.g. with many skyscrapers and/or harbour views).

\begin{figure}
    \centering
    \subfigure[Training data ($64 \times 708$) \label{fig:nyc_data}]{\includegraphics[width=0.43\linewidth]{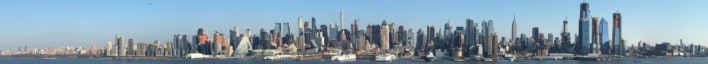}} \hfill
    \subfigure[Sample 1 ($64 \times 640$)\label{fig:nyc_sample1}]{\includegraphics[width=0.40\linewidth]{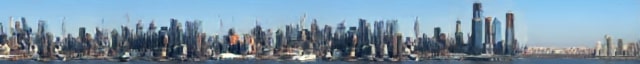}}\\
     \vspace{-5pt}
    \subfigure[Sample 2 ($64 \times 1792$)\label{fig:nyc_long_sample1}]{\includegraphics[width=1.0\linewidth]{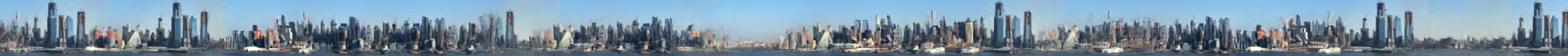}}\\
     \vspace{-5pt}
    \subfigure[Sample 3 ($64 \times 1792$)\label{fig:nyc_long_sample2}]{\includegraphics[width=1.0\linewidth]{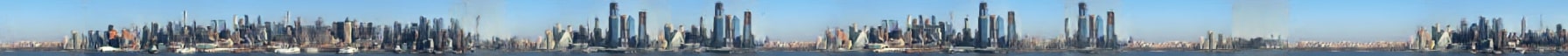}}
    \vspace{-5pt}
    \caption{Generated panoramic views. The model only sees $64 \times 64$ data patches during training. }
    \label{fig:nyc_result}
    \vspace{-12pt}
\end{figure}


\section{Experimental details}
\label{sec:appendix_exp_details}

\subsection{Data collection}

\paragraph{Texture} We download from \url{https://www.textures.com/} 4 texture images. Then we construct datasets from each of them. Specifically, the original image is referred to as ``texture 4x'', and we down-sized this image to ``texture 2x'' and ``texture 1x''. Then random cropping is applied to each of the 4 images to obtain datasets $\{\mathcal{D}_0, \mathcal{D}_1, \mathcal{D}_2\}$ of increasing resolutions, each of them contains $(64, 64)$ image patches of the corresponding resolution. 

\paragraph{Panoramic landscape view}
We download the panoramic landscape images from 

\url{https://en.wikipedia.org/wiki/Skyline#/media/File:10_mile_panorama_of_NYC,_Feb.,_2018.jpg}

We removed the numbers and then cropped the image to remove part of the sky.
We further down-sized the image to have vertical spatial size 64. Then the dataset $\mathcal{D}_0$ contains $(64, 64)$ image patches which are randomly cropped from the full-scale image. 

\paragraph{Satellite world map} The full image is downloaded from: 

\url{https://eoimages.gsfc.nasa.gov/images/imagerecords/74000/74117/world.200408.3x5400x2700.png}

We remove the bottom pixels containing Antarctica in order to remove the bias introduced by the Mercator projection and then down-sized it to half of the size in each spatial dimension. Then multi-scale down-sizing and random cropping are applied to each of the 4 images to obtain datasets $\{\mathcal{D}_0, \mathcal{D}_1, \mathcal{D}_2 \}$ of increasing resolutions, each of them contains $(64, 64)$ image patches of the corresponding resolution. 

\subsection{Network architectures and training details}

\paragraph{Network architecture for $G_0(\cdot)$} For example in training time we start from an $6 \times 6 \times 128$ input of zero values. Note here that all the convolutional layers use stride 1 and no padding. The latent tensors $\{ \bm{z}^k(i, j) \}$ are i.i.d.~Gaussian noises during training, which is sampled when computing a Noisy AdaPixNorm (NAPN) operation. They have the same spatial shape as the input to NAPN but with channel 1, e.g.~if the input tensor has shape $(6, 6, C)$, then the sampled $\bm{z}^k$ tensor will have shape $(6, 6, 1)$. See main text for the math expression of the NAPN operation. We use bi-linear interpolation for up-sampling, to remain consistent we crop out the boundary pixels with edge size 1. This means that a tensor of shape $(H, W, C)$, after this up-sampling, will be transformed to another tensor of shape $(2H-2, 2W-2, C)$.

With input tensor of shape $(H, W, C) = (6, 6, 128)$:

Block 1:
$$(6, 6, 128) \text{ input} \rightarrow \text{NAPN} \rightarrow \text{bi-linear up-sample} \rightarrow \text{conv } 3 \times 3 \rightarrow \text{ReLU} \rightarrow \text{output } (8, 8, 128).$$
Block 2:
$$(8, 8, 128) \text{ input} \rightarrow \text{NAPN} \rightarrow \text{bi-linear up-sample} \rightarrow \text{conv } 3 \times 3 \rightarrow \text{ReLU} \rightarrow \text{output } (12, 12, 128).$$
Block 3:
$$(12, 12, 128) \text{ input} \rightarrow \text{NAPN} \rightarrow \text{bi-linear up-sample} \rightarrow \text{conv } 3 \times 3 \rightarrow \text{ReLU} \rightarrow \text{output } (20, 20, 128).$$
Block 4:
$$(20, 20, 128) \text{ input} \rightarrow \text{NAPN} \rightarrow \text{bi-linear up-sample} \rightarrow \text{conv } 3 \times 3 \rightarrow \text{ReLU} \rightarrow \text{output } (36, 36, 128).$$
Block 5:
$$(36, 36, 128) \text{ input} \rightarrow \text{NAPN} \rightarrow \text{conv } 3 \times 3 \rightarrow \text{ReLU} \rightarrow \text{output } (34, 34, 64).$$
Block 6:
$$(34, 34, 64) \text{ input} \rightarrow \text{NAPN} \rightarrow \text{bi-linear up-sample} \rightarrow \text{conv } 3 \times 3 \rightarrow \text{ReLU} \rightarrow \text{output } (64, 64, 64).$$
These blocks transforms the $(6, 6, 128)$ input tensor to $\bm{f}_0(I_0)$ of shape $(64, 64, 64)$. Lastly a $1\times1$ convolution and a Tanh layer are applied to $\bm{f}_0(I_0)$ to transform it into an image $\bm{x}_0(I_0)$ of shape $(64, 64, 3)$.

Note that batch-normalization (BN) \citep{ioffe2015batch} layers might be added before ReLU but after the conv $3 \times 3$ layers. Since in test time the ``evaluation mode'' of BN is used, then the applied normalization statistics are independent to the current latent tensors, therefore this test-time BN is still a consistent transformation of stochastic processes. 

\paragraph{Network architecture for $G_{l}(\cdot), l \geq 1$} For example with an input $\bm{f}_{l-l}(J_l)$ of shape $(34, 34, 64)$ which is cropped from the output of the last network $\bm{f}_{l-l}(I_{l-1})$:
$$(34, 34, 64) \text{ input} \rightarrow \text{bi-linear up-sample} \rightarrow \text{conv } 3 \times 3 \rightarrow \text{ReLU} \rightarrow \text{output } (64, 64, 64).$$
These layers transforms the $(34, 34, 64)$ input tensor $\bm{f}_{l-l}(J_l)$ to $\bm{f}_l(I_l)$ of shape $(64, 64, 64)$. Lastly a $1\times1$ convolution and a Tanh layer are applied to $\bm{f}_l(I_l)$ to transform it into an image $\bm{x}_l(I_l)$ of shape $(64, 64, 3)$.

\paragraph{Network architecture for $D_{l}(\cdot), l \geq 0$} The discriminators at all layers use the same architecture, and they observes image inputs of the same size (in our case $(64, 64, 3)$).
With an input $\mathbf{x}$ of shape $(64, 64, 3)$, a discriminator computes the scalar output by the following transforms. The architecture for each of the discriminators has the following architecture.

Layer 1:
$$(64, 64, 3) \text{ input} \rightarrow \text{ conv } 3 \times 3, \text{ stride } 2, \text{ zero padding size } 1 \rightarrow \text{LeakyReLU }(0.2) \rightarrow (32, 32, 64)$$
Layer 2:
$$(32, 32, 64) \text{ input} \rightarrow \text{ conv } 3 \times 3, \text{ stride } 2, \text{ zero padding size } 1 \rightarrow \text{LeakyReLU }(0.2) \rightarrow (16, 16, 64)$$
Layer 3:
$$(16, 16, 64) \text{ input} \rightarrow \text{ conv } 3 \times 3, \text{ stride } 2, \text{ zero padding size } 1 \rightarrow \text{LeakyReLU }(0.2) \rightarrow (8, 8, 128)$$
Layer 4:
$$(8, 8, 128) \text{ input} \rightarrow \text{ conv } 3 \times 3, \text{ stride } 2, \text{ zero padding size } 1 \rightarrow \text{LeakyReLU }(0.2) \rightarrow (4, 4, 128)$$
Layer 5:
$$(4, 4, 128) \text{ input} \rightarrow \text{ conv } 4 \times 4, \text{ stride } 1, \text{ no padding} \rightarrow \text{LeakyReLU }(0.2) \rightarrow (1, 1, 128)$$
Layer 6:
$$(1, 1, 128) \text{ input} \rightarrow \text{ conv } 1 \times 1, \text{ stride } 1, \text{ no padding} \rightarrow (1, 1, 1)$$

\paragraph{An important note} The index sets $\{J_l\}$ and $\{I_l\}$ are used only for mathematical rigor, none of the generative networks $G_i(\cdot)$ nor the discriminative networks $D_i(\cdot)$ take these index sets as inputs. 

\paragraph{The hyper-parameters}
We set the $\lambda_l$ parameters associated with the consistency losses $\mathcal{L}_C(G_l)$ as $\lambda_1 = 1000, \lambda_2 = 1000$. As we also use the consistency loss to match between a patch cropped from $\bm{x}_0(I_0)$ and a down-sampled version of $\bm{x}_2(I_2)$, the balancing parameter for this consistency loss is selected as $16000$. Average pooling is used as the down-sampling method $d(\cdot)$. The WGAN-GP losses $\mathcal{L}_{adv}(G_l, D_l)$ use $\lambda=10$ for the gradient penalty. The training schedule is to optimize the parameters of $\{D_l(\cdot)\}$ for 5 iterations then optimize the parameters of $\{G_l(\cdot)\}$. We use Adam optimizer fot both the generators and the discriminators is selected, with learning rate $0.0001$ and the momentum damping rate $\beta_1 = 0.5$. 

\paragraph{Baseline}
We use a PyTorch reimplementation of Spatial GAN and PSGAN which is linked to from the original authors repository.\footnote{\url{https://github.com/zalandoresearch/famos}} 
We perform a grid search over hyperparameter settings for the baseline.  We search find channel sizes for the generator and the discriminator in $[10,20,30,40,50,60,70,80]$ and search input crop sizes in $[128,256,512,1024]$.  We choose the best architecture based on FID score.  For spatial GAN this was a final channel size of $20$ and a input crop size of $256$.  For PSGAN this was a final channel size of 40 and an input crop size of $512$. 

\end{document}